\documentclass{article} %
\usepackage[utf8]{inputenc}

\usepackage{etoolbox}

\NewDocumentCommand{\flexequation}{ m O{} O{} O{} O{} }{%
        \begin{align*}
            #1#3
        \end{align*}#5
}

\usepackage{txie}
\usepackage{algorithm}
\usepackage{algorithmic}

\usepackage{geometry}
\makeatletter
\usepackage{parskip}
\def\thm@space@setup{%
  \thm@preskip=\parskip
  \thm@postskip=\parskip
}
\makeatother

\let\oldparagraph\paragraph
\renewcommand{\paragraph}[1]{\oldparagraph{#1.}}

\usepackage{booktabs}
\usepackage[T1]{fontenc}
\usepackage{lmodern}
\usepackage{microtype}
\usepackage[scaled=.91]{helvet}
\usepackage{inconsolata}

\usepackage{pgfplots}
\pgfplotsset{compat=1.18} %
\usepgfplotslibrary{colormaps}
\usepackage[most]{tcolorbox}
\usepackage{soul}
\usepackage{calc}
\usepackage{ulem}
\usepackage{booktabs}
\usepackage{listings}
\usepackage{threeparttable}
\usepackage{multirow}
\usepackage{makecell}
\usepackage{changepage}

\newenvironment{inlineenum}
  {\begin{enumerate*}[label=\bfseries(\arabic*)]}
  {\end{enumerate*}}

\allowdisplaybreaks

\newcommand{\titlestr}{Breaking the Capability Ceiling of LLM Post-Training\\ by Reintroducing Markov States}

\title{\titlestr{}}
\author{Yurun Yuan\footnotemark[2] \qquad  Tengyang Xie\footnotemark[1]~\footnotemark[2]
\\\\
\footnotemark[2] UW-Madison}

\date{\today}

\usepackage{tabularx}
\definecolor{highlightbgcolor}{HTML}{D0E0F7} 
\definecolor{lightgray}{gray}{0.9}
\newcommand{\metricname}[1]{{\scriptsize \textbf{#1}}}

\newcommand{\smallrow}[1]{\scriptsize \color{gray} #1}

\begin{document}
\maketitle

\def\thefootnote{}
\footnotetext{Email: \texttt{\{\href{mailto:yurun\_yuan@cs.wisc.edu}{{yurun\_yuan}},\href{mailto:tx@cs.wisc.edu}{tx}\}@cs.wisc.edu}}
\def\thefootnote{$*$}
\footnotetext{Corresponding author}
\def\thefootnote{\arabic{footnote}}

\begin{abstract}
Reinforcement learning (RL) has become a standard paradigm for post-training and aligning Large Language Models (LLMs), yet recent evidence suggests it faces a persistent ``capability ceiling'': unlike classical RL systems that discover novel strategies, RL for LLMs often acts as a mere refiner of patterns already latent in pre-trained weights. In this work, we identify a fundamental structural bottleneck: while classical RL relies on compact, informative Markov states, current LLM post-training formulations are tethered to an ever-expanding history of actions.

We revisit a classical principle long central to RL yet absent from LLM post-training: explicit Markov states. Theoretically, we provide rigorous guarantees demonstrating that leveraging estimated Markov states can significantly reduce sample complexity. Empirically, we show that introducing Markov states consistently breaks the performance boundaries of standard RL post-training across a suite of complex logic puzzles. Our findings suggest that moving beyond ``history-as-state'' modeling in favor of structured Markovian representations is essential for unlocking open-ended discovery and genuinely new reasoning capabilities in Generative AI.
\end{abstract}

\section{Introduction}

Reinforcement learning (RL) has emerged as the definitive paradigm for post-training and aligning Large Language Models (LLMs), enabling breakthroughs in complex reasoning, mathematical problem-solving, and agentic behavior \citep{jaech2024openai,guo2025deepseek}. By shifting from static supervised fine-tuning to dynamic environment interaction, RL allows models to explore vast solution spaces. In the prevailing RL post-training paradigm, LLMs are formulated as agents where the action space consists of discrete tokens and the state is defined by the concatenation of all preceding actions \citep{guo2025deepseek}.

Despite these successes, growing evidence suggests that RL primarily functions as a mechanism for sharpening search within regions already reachable by the base model, rather than fundamentally expanding its solution space \citep{yue2025does, wu2025invisible, shao2025spurious, yeo2025demystifying}. While some contemporary studies claim that RL can elicit novel capabilities, these gains typically manifest as modest extensions of the pre-training boundary or are contingent upon dense reward shaping or specialized domain-specific designs \citep{zhang2025interplay,sun2025rl,yuan2025f}. \citet{foster2025good} provide theoretical evidence that significant capability expansion is often computationally prohibitive, as the cost of discovery is lower-bounded by either the exponential complexity of the parameter search space or the inherent limitations of the base model's coverage.

This ``capability ceiling'', however, appears to be a unique artifact of the LLM--RL intersection rather than an inherent limitation of RL itself. In classical RL environments---ranging from robotic manipulation to complex board games---RL has been serving as a powerful discovery engine rather than a mere capability refiner. For example, systems like AlphaZero \citep{silver2017mastering} and MuZero \citep{schrittwieser2020mastering} demonstrated the ability to transcend human knowledge, developing novel strategic patterns and superhuman heuristics that were entirely absent from their initial programming or training data. 
The presence of a performance plateau in RL post-training for LLMs indicates that current formulations may be structurally constrained, necessitating a rethink of foundational assumptions. 

A critical distinction emerges when comparing classical RL with its application to LLMs. In classical RL applications, such as robotics or board games like Go, the Markov states are central: a compact representation that encapsulates all information necessary for optimal future decision-making. In contrast, current LLMs operate over a cumulative sequence of previous tokens, relying on an ever-expanding, inherently noisy history rather than a distilled, Markovian representation. Therefore, we argue that this ``capability ceiling'' is a consequence of the \textit{action-sequence-based} formulation and hypothesize that the reintroduction of the Markov states is the key to unlocking genuinely new reasoning capabilities and improving generalization. 

In this paper, we revisit a classical RL principle, explicit Markov state (estimation), and demonstrate its critical importance for LLM post-training. We provide both theoretical foundations and empirical evidence showing that this simple idea yields significant improvements over traditional history-dependent formulations. Our primary contributions are as follows:
\begin{itemize}[nosep]
\item \textbf{Breaking the Capability Ceiling:} Through extensive benchmarking on a suite of complex logic puzzles, we show that models with explicit Markov states consistently surpass the performance boundaries of traditional RL post-training, achieving high success rates on tasks where history-dependent models plateau or fail.
\item \textbf{Robust Generalization:} We demonstrate that Markov models exhibit superior out-of-distribution (OOD) generalization, effectively solving puzzles with higher structural complexity and search depth than those encountered during training.
\item \textbf{Sample Efficiency Guarantees:} We provide theoretical guarantees demonstrating that Markovian learning achieves significantly lower sample complexity compared to standard action-sequence-based formulations.
\end{itemize}

Taken together, our findings suggest that the path toward artificial general intelligence and open-ended capability growth may require moving beyond ``history-as-state'' modeling in favor of Markovian states that better align with the underlying logic of complex reasoning tasks.

\section{Preliminaries}
\label{sec:preliminaries}

\subsection{Markov Decision Process, Policies, and Value Functions}
\label{sec:mdp-plc-value}

RL provides a framework for sequential decision-making problems where an agent interacts with an environment to maximize cumulative rewards. In the context of Markov Decision Processes (MDPs), which provide the theoretical foundation for RL, we consider an episodic finite-horizon framework. Formally, a horizon-$H$ episodic MDP $M=(H,\Scal, \Acal, \Pcal, r, \rho)$ consists of a (potentially very large) state space $\Scal$, an action space $\Acal$, a probability transition function $\Pcal: \Scal \times \Acal \to \Delta(\Scal)$, a reward function $r: \Scal \times \Acal \to \mathbb{R}$, and an initial state distribution $\rho \in \Delta(\Scal)$. The state space is typically layered such that $\Scal = \Scal_1 \cup \Scal_2 \cup \cdots \cup \Scal_H$, where $\Scal_h$ is the set of states reachable at step $h$. A policy $\pi: \Scal \to \Delta(\Acal)$ maps states to distributions over actions and induces a distribution over trajectories $\tau = (s_1, a_1, \ldots, s_H, a_H)$ and rewards $(r_1, \ldots, r_H)$, where the initial state is sampled as $s_1 \sim \rho$, and for $h=1,\ldots,H$: $a_h \sim \pi(s_h)$, $r_h = r(s_h, a_h)$, and $s_{h+1} \sim \Pcal(s_h, a_h)$. We let $\E_{\tau \sim \pi}[\cdot]$ and $\mathbb{P}_{\tau \sim \pi}[\cdot]$ denote expectation and probability under this process, and $\E_{\pi}[\cdot]$ and $\mathbb{P}_{\pi}[\cdot]$ for brevity when $\tau$ is not explicitly mentioned. 

The expected cumulative reward of a policy $\pi$ is given by $J(\pi) = \E_{\tau \sim \pi}[r(\tau)]$, where $r(\tau) = \sum_{h=1}^H r(s_h,a_h)$. The value function and $Q$-function of $\pi$ is defined as
$
V_h^\pi(s) \coloneqq  \mathbb{E}_\pi \left[ \sum_{\ell=h}^H r_\ell(s_\ell, a_\ell) \mid s_h = s \right], \quad Q_h^\pi(s, a) \coloneqq  \mathbb{E}_\pi \left[ \sum_{\ell=h}^H R_\ell(s_\ell, a_\ell) \mid s_h = s, a_h = a \right].
$
Additionally, the advantage function $A^\pi$ represents the relative benefit of taking a specific action $a$ compared to following the policy $\pi$ on average, defined as:
$
A_h^\pi(s, a) \coloneqq  Q_h^\pi(s, a) - V_h^\pi(s).
$
We denote the optimal policy as $\pi^\star$ (i.e., $\pi^\star \in \text{argmax}_\pi J(\pi)$) and its associated value, $Q$, and advantage functions as $V^\star$, $Q^\star$, and $\Astar$, respectively. 

\subsection{Reinforcement Learning for Language Models}
\label{sec:prelim-rl-llm}

In the context of language models, the model serves as the policy $\pi$, and the input and output of the model maps to the state $s$ and action $a$ respectively. In the single-turn setting where $x \sim \rho$ denotes the input prompt and $y_1, y_2, \dotsc, y_H$ denote the output tokens, we can define $s_1 \coloneqq x$ and $s_h \coloneqq (x,y_1,\dotsc,y_{h-1})$ for $h > 1$, with $a_h \coloneqq y_h$ for $h=1,\dotsc,H$. In the multi-turn setting, which consists of multiple interaction turns $(x^\iter{1}, y_{1:H}^\iter{1})$, $(x^\iter{2}, y_{1:H}^\iter{2})$, and so forth, we can adapt the transition function accordingly. Here, $y_{1:H}^\iter{i}$ is a shorthand notation for the sequence of tokens $y_1^\iter{i}, y_2^\iter{i}, \dotsc, y_H^\iter{i}$ in the $i$-th turn. For instance, if a state-action pair $(s,a)$ contains the complete response for one turn (e.g., in a conversation with three or more turns), where $s = (x^\iter{1}, y_{1:H}^\iter{1}, x^\iter{2}, y_{1:H-1}^\iter{2})$ and $a = y_H^\iter{2}$, the next state would transition to $s' = (x^\iter{1}, y_{1:H}^\iter{1}, x^\iter{2}, y_{1:H}^\iter{2}, x^\iter{3})$.

In standard RL, the objective is to find a policy $\pi$ that maximizes the expected cumulative reward $J(\pi) = \E_{\tau \sim \pi}[r(\tau)]$.
In many practical applications, particularly in the context of large language models, it is beneficial to incorporate a regularization term that encourages the learned policy to stay close to a reference policy $\piref$. This leads to the KL-regularized RL objective \citep{ziebart2008maximum,ziebart2010modeling,neu2017unified,ouyang2022training,xie2024exploratory,yuan2025trajectory}
\begin{align*}
J_{\beta}(\pi) = \E_{\tau \sim \pi}[r(\tau)] - \beta \cdot \E_{\tau \sim \pi}\left[\log\frac{\pi(\tau)}{\piref(\tau)}\right],
\end{align*}
where $\beta > 0$ is a regularization parameter that controls the strength of the penalty $D_{\text{KL}}(\pi \| \piref) = \E_{\tau \sim \pi}\left[\log\frac{\pi(\tau)}{\piref(\tau)}\right]$, known as the Kullback-Leibler divergence. We use $\pi^\star_\beta \in \text{argmax}_\pi J_\beta(\pi)$ to denote the KL-regularized optimal policy. 

Proximal Policy Optimization \citep[PPO;][]{schulman2017proximal} and Group Relative Policy Optimization \citep[GRPO;][]{shao2024deepseekmath} represent the primary algorithmic frameworks currently utilized for reinforcement learning post-training and alignment.
They introduce a clipped surrogate objective to constrain policy updates:
\begin{align*}
\mathcal{J}(\theta) = \E_{(s_h, a_h) \sim \pi_{\theta_{\sf old}}} \Big[ \min\Big( \tfrac{\pi_\theta(a_h\mid s_h)}{\pi_{\theta_{\sf old}}(a_h\mid s_h)} \widehat A_h(s_h,a_h), 
\text{clip}\Big(\tfrac{\pi_\theta(a_h\mid s_h)}{\pi_{\theta_{\sf old}}(a_h\mid s_h)}, 1-\varepsilon, 1+\varepsilon\Big) \widehat A_h(s_h,a_h) \Big) \Big],
\end{align*}
where $\widehat A_h$ is the advantage estimate, and $\varepsilon$ is a hyperparameter. In PPO, the advantage $\widehat A_h$ is typically computed using Generalized Advantage Estimation \citep[GAE;][]{schulman2015high} to estimate the advantage of the KL regularized reward $r(s_h,a_h) - \beta \log\frac{\pi(a_h\mid s_h)}{\piref(a_h\mid s_h)}$.
GRPO is a policy-based method that, in practical implementations for LLMs like DeepSeek-R1, samples $G$ responses $o^\iter{1}, \dots, o^\iter{G}$ for each prompt $x$ and computes advantages by normalizing rewards within each prompt group.
This response-level advantage $\widehat A^\iter{i}$ is then used to replace the step-wise advantage function $\widehat A_h(s_h,a_h)$ in the objective $\mathcal{J}(\theta)$. GRPO objective then accommodates the KL-regularization at the end. GRPO is often considered a simpler alternative to PPO for post-training LLMs. This is partly because PPO typically involves training a separate critic network and incorporates more complex mechanisms for policy updates. In the context of LLMs, the full complexity of PPO might not always be necessary, leading to the adoption of more streamlined policy gradient methods like GRPO.

\section{Reintroducing Markov States to LLM Post-Training}

\subsection{Limits of Current RL for LLMs}

Despite the empirical success of RL in improving the reasoning performance of large language models \citep{guo2025deepseek, jaech2024openai}, it remains debated whether RL can induce capabilities that fundamentally exceed those acquired during pre-training. A growing body of evidence suggests that RL primarily reweights or amplifies reasoning patterns already latent in the base model, rather than creating genuinely novel capabilities \citep{yue2025does, wu2025invisible, shao2025spurious, yeo2025demystifying}. 

Conversely, reports of emergent capabilities under RL typically rely on restrictive training designs \citep{yuan2025f,zhang2025interplay,sun2025rl}.
These mechanisms guide learning toward known solution manifolds, suggesting that the observed gains reflect controlled extrapolation within a limited hypothesis space rather than the discovery of new reasoning trajectories.

Recent work \citep{foster2025good} also provides theoretical evidence for this fundamental boundary. Let $C_{\sf cov}(\pi^\star_\beta)\coloneqq \max_{s,a}\frac{\pi^\star_\beta(y\mid x)}{\pi_{\sf ref}(y\mid x)}$ be the coverage coefficient of the base model $\piref$ w.r.t.~the sequence of tokens, which controls the quality of pass@$k$ performance of the base model.

The hope for emergent capabilities under RL is that: if both the statistical and computational complexity of RL were much smaller than $C_{\sf cov}(\pi^\star_\beta)$ (e.g., \citealt{xie2024exploratory} shows that the statistical complexity can be much smaller than $C_{\sf cov}(\pi^\star_\beta)$ in certain cases), then RL could yield significant gains beyond the base model's pass@$k$ performance. However, \citet{foster2025good} establish the following lower bound on computational complexity:
under the KL-regularized RL objective, achieving a near-optimal policy requires the number of sampling oracle calls (and runtime) $T_{\sf comp}$ to be lower-bounded by
\flexequation{\Omega\left(\min\left\{ C_{\sf cov}(\pi^\star_\beta), \exp\left( \frac{R_{\sf max}}{\beta} \right) \right\}\right)}[][,][,][]
even for the linearly-parameterized softmax policy class.
$R_{\sf max}$ is an upper bound on the reward $r(\tau)$.

This lower bound reveals a strict ``discovery'' bottleneck: for RL to find a near-optimal policy efficiently, an algorithm is forced to either
\begin{inlineenum}
    \item rely on the base model to already cover the optimal response with small $C_{\sf cov}(\pi^\star_\beta)$, which in turn implies that the base model's pass@$k$ performance is already reasonable, or
    \item resort to brute-force exploration over the response space, at a cost that grows exponentially as $e^{R_{\sf max}/\beta}$.
\end{inlineenum}
In particular, if the base model's pass@$k$ performance is poor (i.e., $C_{\sf cov}(\pi^\star_\beta)$ is large), RL is pushed into the exponential-cost regime, making the discovery of truly novel solutions computationally intractable.

Collectively, these findings point to a ceiling of current RL-based post-training paradigms: rather than expanding the model's solution space, RL predominantly sharpens search within regions already accessible to the base model, yielding at most modest extensions beyond its pre-training boundary. This motivates re-examining the foundations of RL for LLMs and casts doubt on whether existing approaches alone can support open-ended capability growth.

\subsection{A Didactic Example}
\label{sec:combination-lock}

\newcommand{\combinationLockPlotSnippet}{%
    \includegraphics[width=\linewidth]{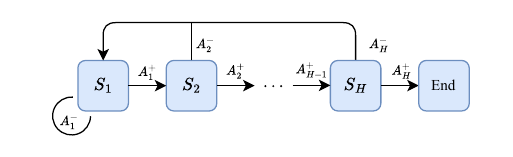}
    \caption{The \emph{Combination Lock} problem with horizon $H$. At each state $S_h$, the correct action $A_h^+$ advances the agent to the next state; the incorrect action $A_h^-$ resets it to the starting position.}
    \label{fig:combination-lock-game}
}

\newcommand{\combinationLockEvalResultPlotSnippet}{
    \includegraphics[width=\linewidth]{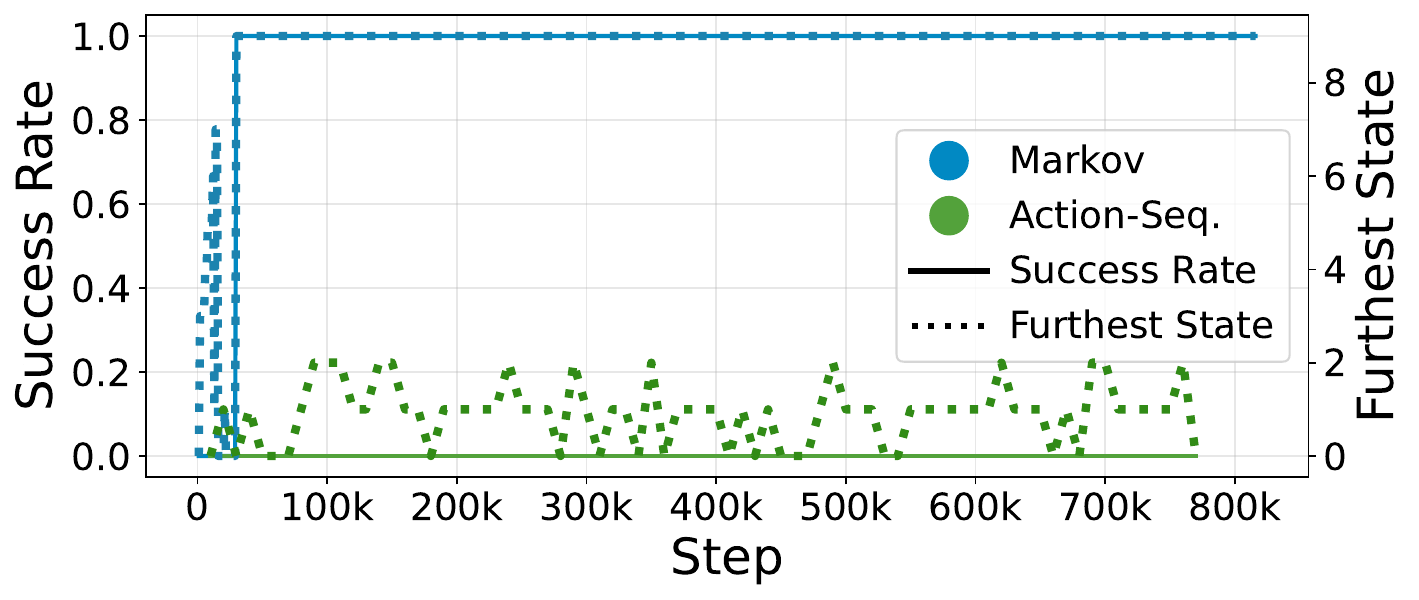}
    \caption{Comparison of Markov agent and action-sequence agent on \emph{Combination Lock} task. We focus on two metrics: the success rate in reaching the final goal state and the furthest state reached before the agent triggers an incorrect action. }
    \label{fig:combination-lock-exp-eval}
}

\begin{figure}[htbp]
  \begin{minipage}[t]{0.48\textwidth}
    \centering
    \combinationLockPlotSnippet
  \end{minipage}
  \hfill %
  \begin{minipage}[t]{0.48\textwidth}
    \centering
    \combinationLockEvalResultPlotSnippet
  \end{minipage}
\end{figure}

We start with an empirical analysis on a didactic task: the \emph{Combination Lock} problem. As illustrated in \cref{fig:combination-lock-game}, this environment consists of $H$ linearly ordered states, $(S_1, S_2, \cdots, S_{H})$, and two discrete actions. 
The agent begins at $S_1$; selecting the correct action $A_h^+$ advances the agent to the next state, while an incorrect choice $A_h^-$ resets it to the starting position. Each transition incurs a reward of $-1$, except for the final goal state, which yields a reward of $0$ and terminates the episode. Consequently, the agent must memorize the sequence of $H$ correct actions $(A_1^+,A_2^+,\cdots,A_{H}^+)$ to reach the goal. 

We instantiate this task with a horizon of $H=10$ and evaluate two multilayer perceptron (MLP)-based agents that approach the problem from distinct modeling perspectives. The first network is Markov-state-based, receiving the encoded representation of the current Markov state $s_h$ as input. The second is action-sequence-based, whose input is the concatenation of all previous actions $(a_1, \dots, a_{h-1})$. Both agents are trained via Deep Q-Learning \citep{mnih2015human} to select the action $a_h$ that maximizes cumulative reward. We evaluate performance using two key metrics: the success rate in reaching the final goal state and the furthest state reached before the agent triggers an incorrect action. As shown in \cref{fig:combination-lock-exp-eval}, the Markov agent successfully memorizes the correct actions and stabilizes within 30k steps, while the action-sequence agent fails to reach the goal even after 800k steps. 

The substantial performance gap between the two paradigms is not surprising. For the Markov agent, the input space coincides with the state space and contains only $H$ distinct values while the action-sequence agent operates over an input space consisting of full action histories, suggesting that incorporating Markov states in the inputs is essential for solving this task efficiently.

\begin{figure*}[t]
    \centering
    \noindent
    \begin{minipage}{0.48\textwidth}
        \begin{algorithm}[H]
            \caption{Action-Sequence Models (\cref{fig:action-seq-lm})}
            \label{alg:action-seq}
            \begin{algorithmic}[1]
                \STATE \textbf{Input:} Initial state $s_1$
                \FOR{$h = 1, \dots, H$}
                    \STATE Sample action $a_h \sim \pi(\cdot \mid s_1, a_{1:h-1})$
                    \STATE Append $a_h$ to history sequence
                \ENDFOR
            \end{algorithmic}
        \end{algorithm}
    \end{minipage}
    \hfill
    \begin{minipage}{0.48\textwidth}
        \begin{algorithm}[H]
            \caption{Markovian Models (\cref{fig:markovian-lm})}
            \label{alg:markov-lm}
            \begin{algorithmic}[1]
                \STATE \textbf{Input:} Initial state $s_1$
                \FOR{$h = 1, \dots, H$}
                    \STATE Sample action $a_h \sim \pi(\cdot \mid {\color{blue}s_h})$
                    \STATE {\color{blue} Update state $s_{h+1} \leftarrow \mathcal{P}(s_h, a_h)$}
                \ENDFOR
            \end{algorithmic}
        \end{algorithm}
    \end{minipage}
\end{figure*}

\begin{figure*}[htbp]
     \centering
     \begin{subfigure}[c]{0.6\textwidth}
         \centering
         \includegraphics[width=\textwidth]{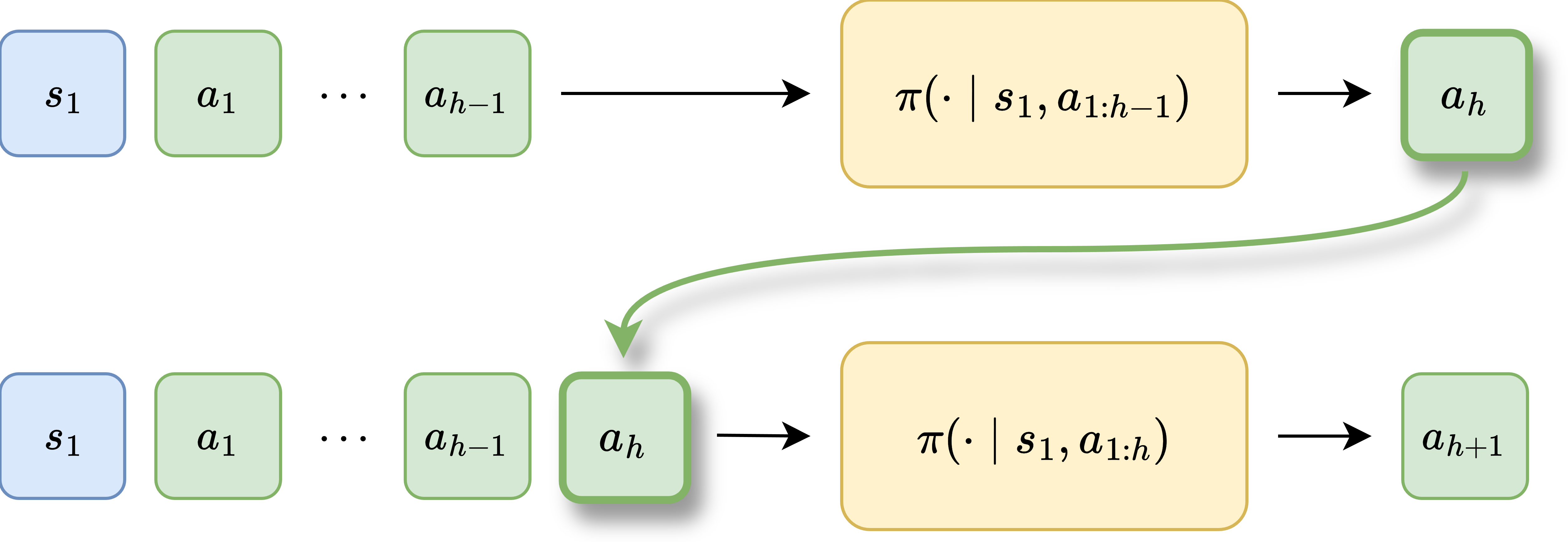}
         \caption{Action-Sequence Models. }
         \label{fig:action-seq-lm}
     \end{subfigure}
     \hfill %
     \begin{subfigure}[c]{0.35\textwidth}
         \centering
         \includegraphics[width=\textwidth]{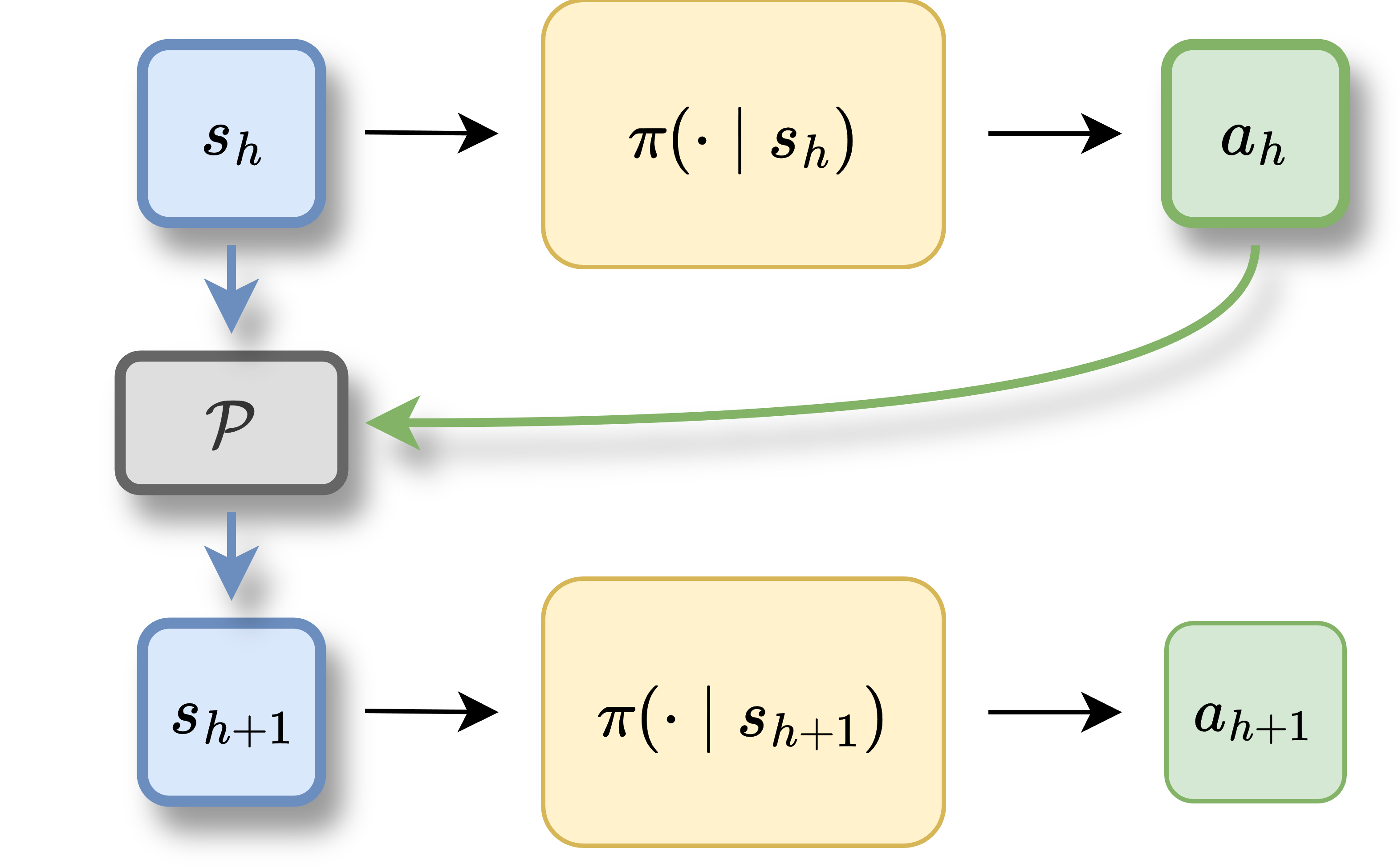}
         \caption{Markov Models. }
         \label{fig:markovian-lm}
     \end{subfigure}
     
     \caption{Comparison of action-sequence-based models and our Markovian Models. For action-sequence models, the new action ($a_h$) is appended to the existing action history and fed back into the model as the input for the subsequent prediction. For Markov models, the next action $a_h$ is combined with the current state \(s_h\) and passed through a state transition function $\mathcal{P}$ to obtain the next state $s_{h+1}$, which is the input for the subsequent prediction.}
     \label{fig:overall_label}
\end{figure*}

\subsection{Markov States in LLM Post-Training}

In contrast to the evidence in \cref{sec:combination-lock}, contemporary LLM post-training practices predominantly adopt an action-sequence-based formulation, where the history of actions is treated as the state, as shown in \cref{alg:action-seq}.
Here, an ``action'' is broadly defined: it may represent a single token, a semantic step toward the final solution, or an iteration of response-refinement.
The pronounced mismatch between prevailing post-training practices and our insights motivates us to rethink about the RL post-training paradigm. 

\paragraph{Markov State Estimation}
We reintroduce explicit Markov states into the LLM training pipeline. As illustrated in \cref{alg:markov-lm}, the newly generated action $a_h$, instead of being simply appended to the previous actions, is combined with the current state \(s_h\) and passed through a state transition function \(\Pcal\). The resulting state \(s_{h+1}\) is then used as the input for the next decision step. By construction, \(s_{h+1}\) preserves all information necessary for future actions while discarding irrelevant noise from the interaction history. In practice, the state transition function may be realized by an environment that internally maintains a Markov state, a rule-based mechanism implementing the transition logic, or a learned model that approximates the underlying transition dynamics. 

\begin{table}[ht]
    \centering
    \small
    \begin{threeparttable}
        \setlength{\tabcolsep}{8pt} %
        \begin{tabular}{lccc}
            \toprule
            \textbf{Approach} & \textbf{Sudoku} & \textbf{Sokoban} & \textbf{Futoshiki} \\
            \midrule
            \textbf{Qwen3-4B} & & & \\
            \textit{Action-seq.} & 92.3 & 2.5 & 0.2 \\
            \rowcolor{highlightbgcolor} \textit{Markov} & 97.1 & 76.1 & 75.0 \\
            \bottomrule
        \end{tabular}
        \caption{Comparison of action-sequence models and Markov models on logical reasoning tasks. Across all tasks, Markov models consistently outperform their action-sequence counterparts by a substantial margin. }
        \label{tab:experiment-highlights}
    \end{threeparttable}
\end{table}

\paragraph{Empirical Evidence}
To empirically validate the advantages of incorporating Markov states, we conduct experiments on a set of synthetic, controllable tasks with well-defined Markov state representations. In particular, we consider a suite of logical reasoning games, including Sudoku, Sokoban, and Futoshiki. For each task, we post-train models using both action-sequence-based and Markov paradigms with the same RL algorithm. We also train a separate state transition estimation model. As summarized in \cref{tab:experiment-highlights}, Markov models consistently outperform their action-sequence counterparts by a substantial margin, even on tasks where action-sequence models exhibit near-zero accuracy. 
We defer full experimental details and comprehensive evaluation results to \cref{sec:experiment}.

\paragraph{Broader Implications and Applications}
The applicability of Markovian language models extends well beyond synthetic benchmarks to a wide range of real-world settings. In many domains, the Markov states are \emph{accessible} during training and our paradigm can easily fit in. To illustrate this broader potential, we outline several representative scenarios:
\begin{inlineenum}
\item \textbf{Coding:} In multi-turn code debugging, the state represents a snapshot of the current codebase together with relevant execution or compiler logs, and evolves through actions such as code edits or test executions. In contrast, an action-sequence-based agent observes only its history of proposed changes, without explicitly reasoning over the resulting code snapshot. \citep{qwen3technicalreport,hui2024qwen2}.
\item \textbf{Mathematical reasoning:} The state consists of the set of established lemmas and intermediate results, with each new inference transitioning the system toward a more complete proof \citep{hubert2025olympiad,chen2025seed}.
\item \textbf{Iterative response refinement:} The state is restricted to the most recent draft, and the transition function $\mathcal{P}$ overwrites the previous version with the refined output. This design enables the model to reason over the current solution while avoiding redundant noise from its own edit history \citep{yuan2025reinforce}.
\end{inlineenum}
Standard action-sequence-based baselines ignore these Markovian structures, while our paradigm suggests that aligning the agent's representation with the efficient underlying Markov structure enables it to solve complex, long-horizon tasks that are otherwise intractable.

\section{Experiments}
\label{sec:experiment}

\begin{table*}[th!]
    \centering
    \small
    \setlength{\tabcolsep}{6.7pt}
    \begin{tabular}{l*{12}{c}} 
        \toprule
        \textbf{Approach} & \multicolumn{6}{c}{\textbf{In Distribution}} & \multicolumn{6}{c}{\textbf{Out Of Distribution}} \\
        \cmidrule(lr){2-7} \cmidrule(lr){8-13}
        & \multicolumn{2}{c}{\textbf{Sudoku}} & \multicolumn{2}{c}{\textbf{Sokoban}} & \multicolumn{2}{c}{\textbf{Futoshiki}} 
        & \multicolumn{2}{c}{\textbf{Sudoku}} & \multicolumn{2}{c}{\textbf{Sokoban}} & \multicolumn{2}{c}{\textbf{Futoshiki}} \\
        \cmidrule(lr){2-3} \cmidrule(lr){4-5} \cmidrule(lr){6-7} \cmidrule(lr){8-9} \cmidrule(lr){10-11} \cmidrule(lr){12-13}
        & \metricname{Avg} & \metricname{Pass} 
        & \metricname{Avg} & \metricname{Pass} 
        & \metricname{Avg} & \metricname{Pass} 
        & \metricname{Avg} & \metricname{Pass} 
        & \metricname{Avg} & \metricname{Pass} 
        & \metricname{Avg} & \metricname{Pass} \\
        \midrule
        \textbf{Qwen3-4B} & & & & & & & & & & & & \\
        \textit{Action-sequence} & & & & & & & & & & & & \\[-0.6ex]
        \smallrow{\hspace{0.8em} + SFT Warm-up} & \smallrow{16.1} & \smallrow{98.0} & \smallrow{0.2} & \smallrow{22.0} & \smallrow{0.0} & \smallrow{5.0} & \smallrow{3.1} & \smallrow{64.0} & \smallrow{0.0} & \smallrow{1.0} & \smallrow{0.0} & \smallrow{0.0} \\[-0.6ex]
        \rowcolor{lightgray} \hspace{0.5em} + RL post-training & 93.5 & 97.0 & 2.3 & 4.0 & 0.1 & 7.0 & 69.2 & 82.0 & 0.0 & 0.0 & 0.0 & 0.0 \\
        \textit{Markov} & & & & & & & & & & & & \\[-0.6ex]
        \smallrow{\hspace{0.8em} + SFT Warm-up} & \smallrow{34.2} & \smallrow{100.0} & \smallrow{0.4} & \smallrow{28.0} & \smallrow{0.1} & \smallrow{11.0} & \smallrow{8.7} & \smallrow{86.0} & \smallrow{0.0} & \smallrow{3.0} & \smallrow{0.0} & \smallrow{0.0} \\[-0.6ex]
        \rowcolor{highlightbgcolor} \hspace{0.5em} + RL post-training & 97.1 & 98.0 & 76.1 & 81.0 & 75.0 & 85.0 & 77.8 & 82.0 & 31.6 & 37.0 & 42.6 & 53.0 \\
        \textit{State-action-sequence} & & & & & & & & & & & & \\[-0.6ex]
        \smallrow{\hspace{0.8em} + SFT Warm-up} & \smallrow{8.6} & \smallrow{96.0} & \smallrow{1.4} & \smallrow{61.0} & \smallrow{0.3} & \smallrow{20.0} & \smallrow{1.9} & \smallrow{62.0} & \smallrow{0.2} & \smallrow{15.0} & \smallrow{0.0} & \smallrow{1.0} \\[-0.6ex]
        \rowcolor{lightgray} \hspace{0.5em} + RL post-training & 91.1 & 96.0 & 57.4 & 67.0 & 44.4 & 55.0 & 71.2 & 82.0 & 30.2 & 34.0 & 16.9 & 21.0 \\
        \midrule
        \textbf{Qwen2.5-3B-It} & & & & & & & & & & & & \\
        \textit{Action-sequence} & & & & & & & & & & & & \\[-0.6ex]
        \smallrow{\hspace{0.8em} + SFT Warm-up} & \smallrow{0.3} & \smallrow{23.0} & \smallrow{0.5} & \smallrow{37.0} & \smallrow{6.6} & \smallrow{94.0} & \smallrow{0.0} & \smallrow{0.0} & \smallrow{0.0} & \smallrow{3.0} & \smallrow{0.3} & \smallrow{28.0} \\[-0.6ex]
        \rowcolor{lightgray} \hspace{0.5em} + RL post-training & 0.0 & 0.0 & 1.0 & 1.0 & 61.3 & 84.0 & 0.0 & 0.0 & 0.0 & 0.0 & 24.9 & 56.0 \\
        \textit{Markov} & & & & & & & & & & & & \\[-0.6ex]
        \smallrow{\hspace{0.8em} + SFT Warm-up} & \smallrow{20.0} & \smallrow{99.0} & \smallrow{0.2} & \smallrow{14.0} & \smallrow{8.5} & \smallrow{98.0} & \smallrow{2.9} & \smallrow{75.0} & \smallrow{0.0} & \smallrow{1.0} & \smallrow{0.3} & \smallrow{26.0} \\[-0.6ex]
        \rowcolor{highlightbgcolor} \hspace{0.5em} + RL post-training & 86.0 & 94.0 & 89.7 & 93.0 & 79.8 & 96.0 & 56.4 & 68.0 & 66.9 & 72.0 & 28.3 & 67.0 \\
        \textit{State-action-sequence} & & & & & & & & & & & & \\[-0.6ex]
        \smallrow{\hspace{0.8em} + SFT Warm-up} & \smallrow{22.4} & \smallrow{100.0} & \smallrow{0.6} & \smallrow{41.0} & \smallrow{16.6} & \smallrow{100.0} & \smallrow{2.4} & \smallrow{71.0} & \smallrow{0.0} & \smallrow{1.0} & \smallrow{1.1} & \smallrow{60.0} \\[-0.6ex]
        \rowcolor{lightgray} \hspace{0.5em} + RL post-training & 83.0 & 90.0 & 43.6 & 50.0 & 67.4 & 94.0 & 57.1 & 69.0 & 20.4 & 23.0 & 25.2 & 75.0 \\
        \bottomrule
    \end{tabular}
    \caption{Performance comparison of different approaches. We sample $128$ solutions per question and report \textbf{Avg@128} and \textbf{Pass@128}. }
    \label{tab:markov-actseq}
\end{table*}
In this section, we report comprehensive experiments and analyses of Markovian learning and comparison to its counterparts.

\paragraph{Tasks and Datasets}
\label{par:exp-datasets}
To accurately obtain Markov states in LLM reasoning, we implement three synthetic logical tasks from Reasoning-Gym \citep{stojanovski2025reasoninggymreasoningenvironments}: Sudoku, Sokoban, and Futoshiki. These grid-based puzzles challenge a model’s capacity for logical analysis, constraint satisfaction, and spatial reasoning. Crucially, the configuration of the board at any given step serves as a fully observable Markov state; every discrete action deterministically updates this configuration to form the subsequent state, yielding an explicit state trajectory for training and analysis. 

\paragraph{Models and Training Pipelines}
\label{par:exp-impl-training}
We use $\pimkv$ to denote the Markov models and $\piactseq$ to denote the action-sequence models. 
\cref{sec:examples} provides illustrative examples of how models operate on these tasks. 

Experiments are conducted with Qwen3-4B \citep{qwen3technicalreport} and Qwen2.5-3B-Instruct \citep{qwen2.5}, training a separate model for each task. For each model, we first perform a brief task-specific SFT warm-start stage to establish task understanding and output formatting. We then post-train with GRPO \citep{shao2024deepseekmath} in an interactive setting, where the agent acts in the true environment  with ground-truth transition dynamics $\mathcal{P}^\star$. The agent receives a sparse terminal reward: $1$ for solving the instance and $0$ otherwise. In addition, we train a state prediction model $\pistatetrans$ based on Qwen2.5-3B-Instruct via SFT to predict the next state $\widehat s_{h+1}$ from the current state and action. At test time, $\pistatetrans$ replaces the environment $\mathcal{P}^\star$, enabling deployment without environment access.

\paragraph{Implementation Details}
\label{par:exp-impl-details}
We implement our methods and baselines in the rLLM framework \citep{rllm2025}, largely following the recommended hyperparameter settings. 

We intentionally require the models to output only the next action without chain-of-thought. This is because the base model is natively trained to solve these puzzles end-to-end; when allowed to reason step by step, it often behaves like an implicit transition model, forecasting future board states inside its reasoning trace (see \cref{app:disable-think-mode} for an illustration). This behavior undermines the goal of decomposing the problem into progressive steps. By constraining the model to outputting only the action, we delegate state-transition computation to an external state prediction model, ensuring that the policy conditions on an explicit next state rather than implicitly inferring it during generation.

\subsection{Main Results}

\newcommand{\passAtKEvaluationPlotSnippet}{\includegraphics[width=1\linewidth]{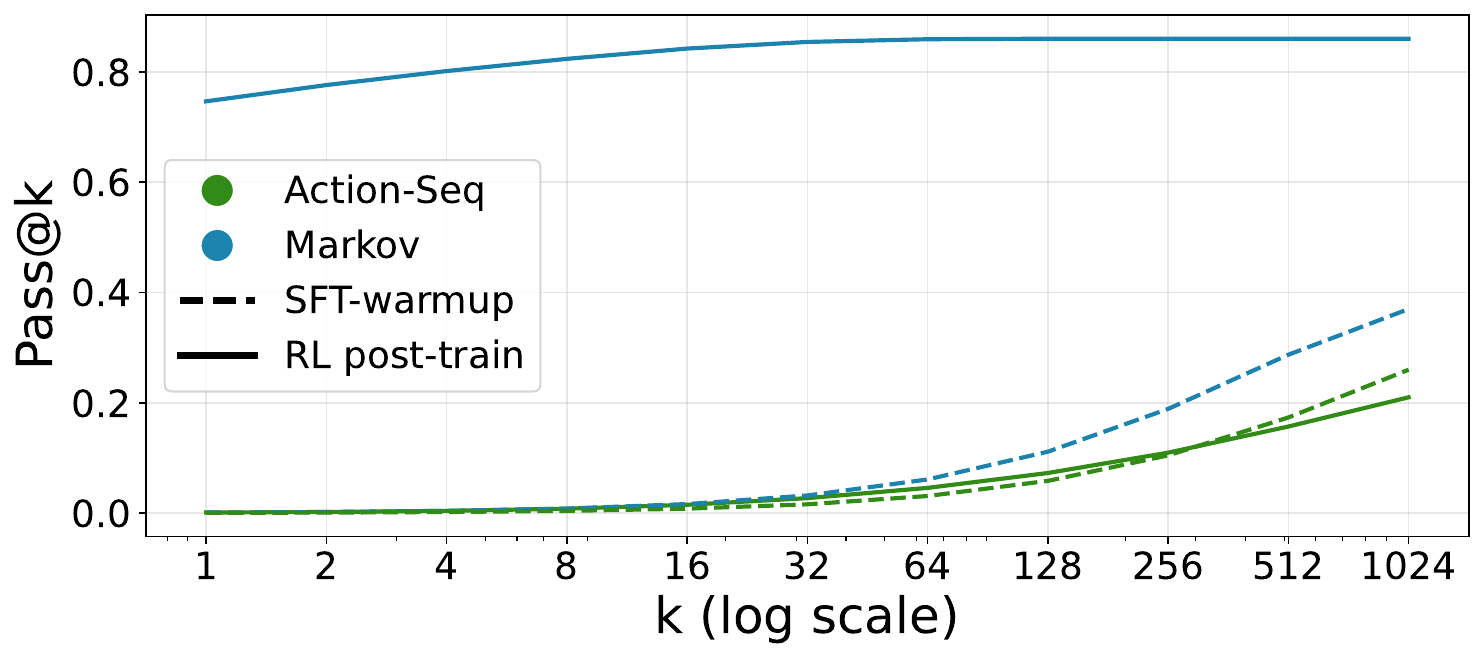}
    \caption{Pass@$k$ accuracy for Qwen3-4B-based models on Futoshiki. While action-sequence models rarely improve SFT Pass@$k$, Markov models consistently surpass their base models' limits.}
    \label{fig:pass-at-k}
}

\newcommand{\trainingRewardPlotSnippet}{
\includegraphics[width=1\linewidth]{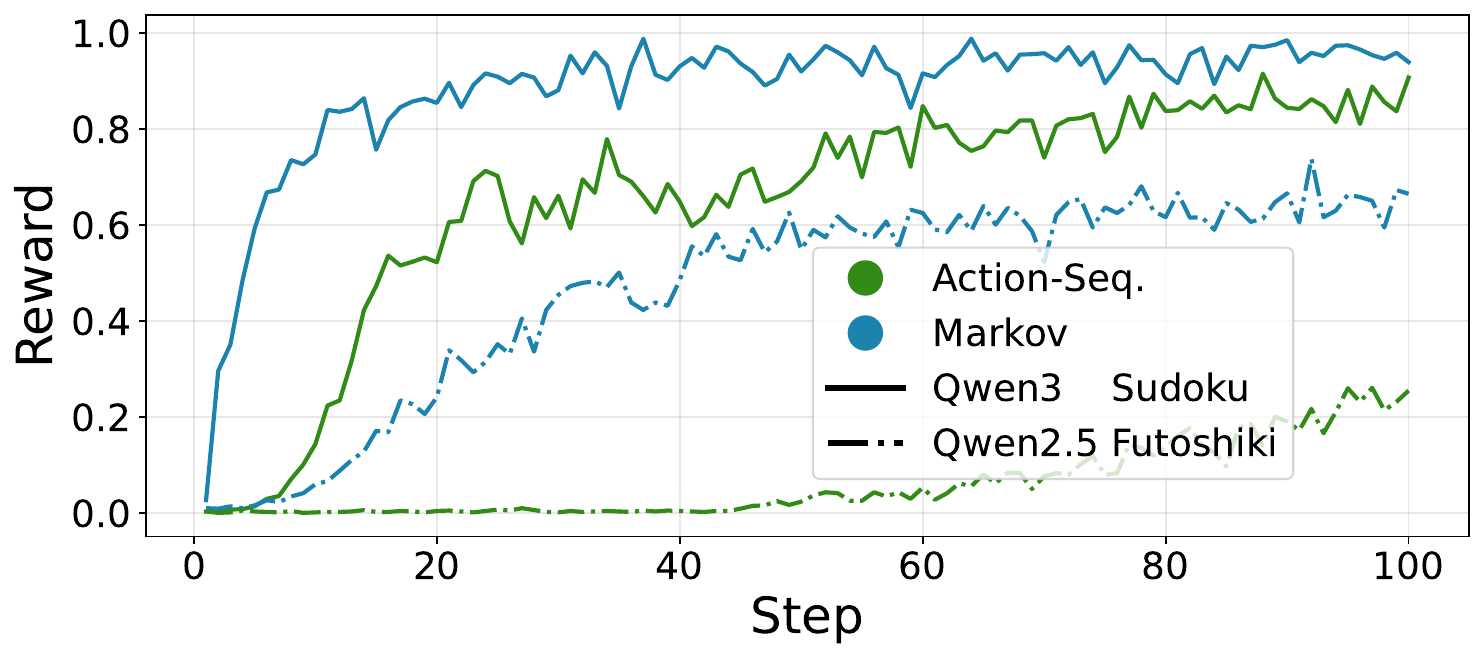}
    \caption{Training reward curves. Markov models reach higher rewards in fewer training steps, providing empirical evidence of lower sample complexity.}
    \label{fig:training-reward}
}

\begin{figure}[htbp]
    \begin{minipage}[t]{0.48\textwidth}
        \centering
        \passAtKEvaluationPlotSnippet
    \end{minipage}
    \hfill
    \begin{minipage}[t]{0.48\textwidth}
        \centering
        \trainingRewardPlotSnippet
    \end{minipage}
\end{figure}

We compare the performance of $\pimkv$ and $\piactseq$ in \cref{tab:markov-actseq} on two evaluation settings:
\begin{inlineenum}
    \item in-distribution (ID) tests, matched to the training set in difficulty and complexity, and 
    \item out-of-distribution (OOD) benchmarks, which are harder than training and typically require more decision steps, thereby probing generalization to greater reasoning depth. 
\end{inlineenum}
For each question, we sample $k=128$ solutions at temperature $1.0$ and report the arithmetic mean of success across all samples, denoted as \textbf{Avg@128}, and the probability that at least one of the $128$ samples is correct, denoted as \textbf{Pass@128}. 

Across both settings, $\pimkv$ consistently improves one-shot accuracy and Pass@$128$. The gains are most pronounced on challenging tasks where $\piactseq$ attains near-zero performance—for example, Sokoban and Futoshiki with Qwen3-4B, and Sudoku and Sokoban with Qwen2.5-3B-Instruct. On OOD benchmarks, Markov models generalize strongly, outperforming action-sequence models on every task.

We further examine how the models' capability boundaries shift by analyzing Pass@$k$ accuracy as $k$ scales. We use the results of Qwen3-4B on Futoshiki as an example and show it in \cref{fig:pass-at-k}. We find that action-sequence models fail to extend or even maintain the Pass@$k$ of SFT models, a result consistent with the findings from \citet{yue2025does}. In contrast, Markov models break through the capability boundaries of the base models, remarkably extending Pass@$k$. Evaluation on other tasks is deferred to \cref{sec:pass-k-full}.

We also observe faster training-time convergence for Markov models. Using Sudoku with Qwen3-4B and Futoshiki with Qwen2.5-3B-Instruct as representative cases, \cref{fig:training-reward} shows that $\pimkv$ reaches higher rewards in fewer training steps, providing empirical evidence of lower sample complexity.

\subsection{Factors Behind the Success of Markov Models}

To understand why $\pimkv$ consistently outperforms $\piactseq$, we decompose their differences into two factors:
\begin{inlineenum}
    \item explicit access to Markov states $s_h$, and
    \item the use of a Markovian structure. 
\end{inlineenum}
To isolate these effects, we introduce an intermediate baseline, the state-action-sequence model $\pistactseq$, which conditions on the full history of states and actions. Concretely, at step $h$, $\pistactseq$ predicts the next action $a_h$ given $(\{(s_i,a_i)\}_{i=1}^{h-1}, s_h)=(s_1, a_1, s_2, a_2, \cdots, a_{h-1}, s_h)$. Unlike $\piactseq$, $\pistactseq$ has access to the true state sequence; however, it still predicts actions from the entire trajectory and is therefore not Markovian. As a result, comparing $\piactseq$ to $\pistactseq$ quantifies the benefit of exposing Markov states, while comparing $\pistactseq$ to $\pimkv$ isolates the additional gains attributable to the Markov property.

In the following, we show that access to Markov states is essential for effective RL learning, and that enforcing a Markov decision structure further improves performance.

\paragraph{Explicit Markov State Conditioning}
The comparison of $\piactseq$ and $\pistactseq$ in \cref{tab:markov-actseq} shows that state-action-sequence models substantially improve logical-reasoning performance, breaking the performance ceiling of action-sequence-based RL training. By conditioning on an explicit state provided by an external transition model, it no longer needs to reconstruct the current board configuration implicitly in its latent space, offloading the burden of state tracking and prediction. During training, a clear state representation enables rewards to be correctly attributed to state-action pairs, rather than to entangled or partially inferred trajectories. At test time, using an explicit state reduces brittleness caused by noisy internal state estimates and makes the policy’s input unambiguous. 

\paragraph{Markovian Property}
Given access to Markov states, enforcing a Markovian policy structure further improves training efficiency. As shown in \cref{tab:markov-actseq}, Markov models consistently outperform state-action-sequence models. This empirical finding is in line with theoretical conclusion that Markov learning has exponentially lower sample complexity. 

While it may seem counterintuitive that state-action-sequence models perform worse despite having more input information, this extra context is often a liability rather than an asset. Under the Markov assumption, the current state has encoded all past information relevant to future decision-making. Consequently, the additional historical context provided to state-action-sequence models is redundant and may introduce spurious correlations that complicate learning. 
To further probe what the state-action-sequence model actually uses after training, we present a controlled ablation in \cref{sec:sas-model-depend-cur-state} showing that state-action-sequence models rely primarily on the current state: removing access to the current state causes performance to collapse, while retaining only the last state preserves a substantial fraction of accuracy.

\subsection{The RL Challenge Addressed by Markov Models}
\label{sec:exp-markov-generalization}
The challenges of RL can be largely summarized as the need to balance \textit{exploration} of unknown actions, \textit{credit assignment} under delayed and sparse rewards, and \textit{generalization} from limited experience to large or unseen state spaces \citep{sutton1998reinforcement,kaelbling1996reinforcement}. Among these challenges, generalization critically depends on how the agent represents and models the environment state. In this section, we will show that Markov models greatly improve generalization. 

To isolate generalization from other factors, we introduce two diagnostic variants: $\pimkv^{\Astar}$ and $\piactseq^{\Astar}$. They follow the same training pipeline as $\pimkv$ and $\piactseq$, except that
\begin{inlineenum}
    \item we remove the SFT warm-up stage and directly train on the base models, and
    \item we replace the estimated response-level advantage $\widehat A^\iter{i}$ used in GRPO objective (see \cref{sec:prelim-rl-llm}) with the ground-truth optimal advantage $\Astar$ defined in \cref{sec:mdp-plc-value}. 
\end{inlineenum}
Practically, $\Astar$ can be obtained using task-dependent rule-based algorithms, detailed in \cref{sec:prm-astar-comp}. 
In this way, the credit-assignment difficulty during training is largely removed through providing $\Astar$ and converting the original sparse reward into a dense, per-step learning signal. Exploration is also controlled: the final $\mathsf{softmax}$ layer of the language model functions as Boltzmann exploration mechanism, and we use the same sampling temperature across all models. Under these controls, remaining performance differences primarily reflect differences in generalization. 

As reported in \cref{tab:markov-astar}, $\pimkv^{\Astar}$ consistently outperforms $\piactseq^{\Astar}$ across models and tasks. The Markov property enables the model to treat distinct action histories as equivalent whenever they induce the same state, thereby promoting generalization to unseen problems. In contrast, action-sequence-based models must explicitly learn this equivalence; previously unseen action sequences may induce erroneous predictions even when they correspond to the same underlying Markov state. We defer the full results in \cref{sec:eval-generalization-full}. 

\begin{table}[th!]
    \centering
    \small
    \begin{threeparttable}
    \setlength{\tabcolsep}{5pt} %
    \begin{tabular}{lccc} 
        \toprule
        \textbf{Approach} & \textbf{Sudoku} & \textbf{Sokoban}\tnote{1} & \textbf{Futoshiki} \\
        \midrule
        \textbf{Qwen3-4B} & & & \\
        \hspace{0.5em} \textit{Action-seq. w/ $A^*$} & 90.8 & 18.2 & 54.8 \\
        \rowcolor{highlightbgcolor}\hspace{0.5em} \textit{Markov w/ $A^*$} & 97.8 & 33.0 & 64.8 \\
        \midrule
        \textbf{Qwen2.5-3B-It} & & & \\
        \hspace{0.5em} \textit{Action-seq. w/ $A^*$} & 28.9 & 0.2 & 44.1 \\
        \rowcolor{highlightbgcolor}\hspace{0.5em} \textit{Markov w/ $A^*$} & 83.4 & 94.2 & 54.2 \\
        \bottomrule
    \end{tabular}
    \begin{tablenotes}
        \footnotesize
        \item[1] Given the capability constraints of the base models, we use lower-complexity Sokoban tasks in this section. 
    \end{tablenotes}
    \caption{Performance comparison of $\pimkv^{\Astar}$, $\piactseq^{\Astar}$, and $\pistactseq^{\Astar}$ on in-distribution benchmarks.}
    \label{tab:markov-astar}
    \end{threeparttable}
\end{table}

\subsection{Additional Results}
Beyond our main results, we provide additional analyses in \cref{sec:addition-results}. We report training-time reward curves demonstrating faster convergence of Markov models (\cref{sec:train-success-rate}). We also conduct an ablation on the fraction of SFT steps used for policy initialization, showing that Markov models achieve higher rewards with fewer SFT steps (\cref{sec:ablation-sft-warmup-degree}). Finally, we examine the role of Markov states in purely supervised settings and find that the Markov property plays a limited role under SFT alone (\cref{sec:markov-pure-sft}).

\section{Theoretical Analysis}
\label{sec:theory}
In this section, we rigorously prove that the introduction of Markov states enables the models to achieving higher performance with lower sample complexity. 
Throughout this section, we assume the underlying environment is deterministic, governed by a ground-truth transition function $s_{h+1} = \Pcal(s_h,a_h)$.\footnote{This assumption naturally holds for many reasoning tasks, including the logical puzzles in our experiments as well as settings like code editing and mathematical proof construction, where actions deterministically update the problem state.} We also assume the reward is bounded $r_h \in [0, 1]$ for all $h$.
Two learning paradigms are formalized as below.
\begin{enumerate}[(1)]
    \item \textit{Action-sequence-based learning}: The policy conditions on the initial state and the action history: $a_h \sim \pi(\cdot \mid s_1, a_{1:h-1})$, where $a_{1:h-1}$ is a shorthand notation for the sequence $a_1, a_2, \cdots,a_{h-1}$. 
    \item \textit{Approximate Markovian learning}: The agent has access to an approximate transition function $\widehat\Pcal: \Scal \times \Acal \to \Scal$ that estimates the true dynamics $\Pcal^\star$. At each step $h$, the agent observes the approximate Markov state $\widehat{s}_h \leftarrow \widehat\Pcal(\widehat s_{h-1},a_{h-1})$ (when $h=2$, $\widehat s_1\coloneqq s_1$), and selects action $a_h \sim \pi(\cdot \mid \widehat{s}_h)$. 
    This setting captures practical scenarios where Markov states are approximately recovered via a learned world model or rule-based state extractor.
\end{enumerate}

We consider a general policy optimization framework applicable to both learning paradigms. Specifically, for each $t$ in $\{1, 2, \dotsc,T\}$, the learning algorithm updates the current policy $\pi^\iter{t}$ using an approximated advantage $\widehat A^{\pi^\iter{t}}$, yielding $\pi^\iter{t+1}$. This algorithmic framework is sketched in \cref{alg:general_policy_optimization}. We make the following assumption on the resulting policies. 

\begin{assumption}[Optimization error]
    For action-sequence-based models, 
    \flexequation{\frac{1}{T}\sum_{t=1}^{T} \sum_{h=1}^{H} \E_{\pi^\star}\left[\widehat A^{\pi^\iter{t}}(s_1,a_{1:h})\right] \leq H\varepsilon_{\sf opt}}[][;][;][]
    and for approximate Markovian learning, 
    \flexequation{\frac{1}{T}\sum_{t=1}^{T} \sum_{h=1}^{H} \E_{\pi^\star}\left[\widehat A^{\pi^\iter{t}}(\widehat s_h,a_h)\right] \leq H\varepsilon_{\sf opt}}[][.][.][]
\label[assumption]{asm:opt-error}
\end{assumption}

As detailed in \cref{sec:opt-error-asm-discussin}, when PPO-style algorithms are employed, the optimization error satisfies $\varepsilon_{\sf opt} =\Ocal\left(\sqrt{1/T}\right)$ . 

We next bound the error of the estimated advantage function $\widehat A^{\pi^\iter{t}}(x_h,a_h)$, as stated in \cref{asm:def-eps-stat}. 

\begin{assumption}[Advantage estimation error]
    We assume \flexequation{\E_{\pi^\iter{t}}\left[\left(\widehat A^{\pi^\iter{t}}(x_h,a_h) - A^{\pi^\iter{t}}(x_h,a_h)\right)^2\right] \leq \varepsilon_{\sf stat}^2}
    for all $h$, where $x_h = (s_1,a_1,\dotsc,a_{h-1})$ for the action-sequence-based learning and $x_h = \widehat s_h$ for approximate Markovian learning. 
\label[assumption]{asm:def-eps-stat}
\end{assumption}

We further define the state-action occupancy as $d^\pi_h(s,a)=\E_{\tau \sim \pi}\left[ \mathbb{I}\left[ s_h=s,a_h=a\right]\right]$ and overload the notation for action-sequence occupancy: $d^\pi_h(s'_1,a'_{1:h})=\E_{\tau \sim \pi}\left[\mathbb{I}\left[ s_1=s'_1,a_{1:h}=a'_{1:h} \right]\right]$. Under these assumptions, we derive the performance guarantee for action-sequence-based learning in \cref{prop:non-markov}. 
\begin{proposition}[Performance guarantee of action-sequence-based learning]
\label[proposition]{prop:non-markov}
For the paradigm of action-sequence-based learning, suppose \cref{asm:opt-error} and \cref{asm:def-eps-stat} hold, then we have 
\flexequation{J(\pi^\star) - \max_t J(\pi^\iter{t}) \leq H \varepsilon_{\sf opt} + H \sqrt{\max_{t,h} \E_{\pi^\iter{t}}\left[ \left( \frac{d_h^{\pi^\star}(s_1,a_{1:h})}{d_h^{\pi^\iter{t}}(s_1,a_{1:h})} \right)^2 \right]}\varepsilon_{\sf stat}}[][.][.][]
\end{proposition}

For approximate Markovian learning, we additionally assume the accuracy of the learned transition model, formalized in \cref{asm:state-transition-accuracy}, and derive the performance guarantee in \cref{prop:app-markov}. 
\begin{assumption}[State transition model accuracy]
    we assume approximate transition model satisfies 
    \flexequation{\Pr\left[\widehat\Pcal(s,a) \neq \Pcal(s,a)\right] \leq \varepsilon_{\Pcal}}
    for all $s\in\Scal,a\in \Acal$. 
    \label[assumption]{asm:state-transition-accuracy}
\end{assumption}

\begin{proposition}[Performance guarantee of approximate Markovian
learning]
\label[proposition]{prop:app-markov}
For the paradigm of approximate Markovian learning, suppose \cref{asm:opt-error}, \cref{asm:def-eps-stat}, and \cref{asm:state-transition-accuracy} hold, then we have 
\flexequation{J(\pi^\star) - \max_t J(\pi^\iter{t}) \leq H \varepsilon_{\sf opt} + H \sqrt{\max_{t,h} \E_{\pi^\iter{t}}\left[ \left( \frac{d_h^{\pi^\star}(s_h,a_h)}{d_h^{\pi^\iter{t}}(s_h,a_h)} \right)^2 \right]}\varepsilon_{\sf stat} + 2 H^3 \varepsilon_{\Pcal}}[][.][.][]
\end{proposition}
The proof of \cref{prop:non-markov} and \cref{prop:app-markov} is deferred to \cref{sec:proof}. 

Comparing \cref{prop:non-markov} and \cref{prop:app-markov} highlights the key benefit of introducing an approximate transition model. 
In the action-sequence-based learning (\cref{prop:non-markov}), the density ratio $\frac{d_h^{\pi^\star}(s_1,a_{1:h})}{d_h^{\pi^\iter{t}}(s_1,a_{1:h})}$ is defined over full action histories, whose space grows exponentially with the horizon $H$. Bounding this ratio in the worst case requires the learning policy $\pi^\iter{t}$ to cover an exponentially large history space, leading to prohibitive sample complexity (essentially scaling as $|\Acal|^H$). This is consistent with the computational lower bound of \citet{foster2025good}: the coverage coefficient $C_{\sf cov}(\pi^\star_\beta)$ captures the same fundamental difficulty of covering the optimal policy over an exponentially large response space.

In contrast, the bound for approximate Markovian learning (\cref{prop:app-markov}) effectively depends on the density ratio of the \emph{true} underlying Markov states: $\frac{d_h^{\pi^\star}(s_h, a_h)}{d_h^{\pi^\iter{t}}(s_h, a_h)}$, despite the agent operating on approximate states $\widehat s_h$. Provided that the true task has a compact state structure (e.g., polynomial in $H$), this ratio is much easier to bound, implying an exponential reduction in variance. \cref{sec:theory-comblock-comparison} illustrates this distinction using \emph{Combination Lock} as an example task. 

This improvement comes at the cost of an additive bias term $\Ocal\left(H^3 \varepsilon_{\Pcal}\right)$, which is polynomial in the horizon and controlled by the transition model accuracy. This presents a favorable trade-off: \emph{by paying a polynomial price for approximate state transitions, we avoid the potentially exponential coverage cost inherent in action-sequence-based learning.} This is particularly significant in light of the computational barrier established by \citet{foster2025good}, where action-sequence RL is forced to pay either the coverage cost $C_{\sf cov}(\pi^\star_\beta)$ or the exponential exploration cost $e^{R_{\sf max}/\beta}$---introducing Markov states sidesteps this bottleneck by reducing the effective coverage requirement from action histories to the compact state space.

\section{Conclusion}
In this work, we reintroduce Markov states into LLM post-training and demonstrate their potential to overcome the performance plateau of contemporary post-training paradigms. We hope this perspective motivates future work to incorporate Markovian structure into real-world, complex tasks, paving the way toward more scalable and open-ended capability growth in generative AI.

\section*{Acknowledgements}
We acknowledge support of the DARPA AIQ Award. This work used the DeltaAI system at the National Center for Supercomputing Applications [award OAC 2320345] through allocation CIS251426 from the Advanced Cyberinfrastructure Coordination Ecosystem: Services \& Support (ACCESS) program, which is supported by U.S.~National Science Foundation grants {\sf\#}2138259, {\sf\#}2138286, {\sf\#}2138307, {\sf\#}2137603, and {\sf\#}2138296.

\bibliographystyle{plainnat}
\bibliography{ref}

@article{guo2025deepseek,
  title={Deepseek-r1: Incentivizing reasoning capability in llms via reinforcement learning},
  author={Guo, Daya and Yang, Dejian and Zhang, Haowei and Song, Junxiao and Zhang, Ruoyu and Xu, Runxin and Zhu, Qihao and Ma, Shirong and Wang, Peiyi and Bi, Xiao and others},
  journal={arXiv preprint arXiv:2501.12948},
  year={2025}
}

@article{jaech2024openai,
  title={Openai o1 system card},
  author={Jaech, Aaron and Kalai, Adam and Lerer, Adam and Richardson, Adam and El-Kishky, Ahmed and Low, Aiden and Helyar, Alec and Madry, Aleksander and Beutel, Alex and Carney, Alex and others},
  journal={arXiv preprint arXiv:2412.16720},
  year={2024}
}

@article{yue2025does,
  title={Does reinforcement learning really incentivize reasoning capacity in llms beyond the base model?},
  author={Yue, Yang and Chen, Zhiqi and Lu, Rui and Zhao, Andrew and Wang, Zhaokai and Song, Shiji and Huang, Gao},
  journal={arXiv preprint arXiv:2504.13837},
  year={2025}
}

@article{wu2025invisible,
  title={The Invisible Leash: Why RLVR May or May Not Escape Its Origin},
  author={Wu, Fang and Xuan, Weihao and Lu, Ximing and Liu, Mingjie and Dong, Yi and Harchaoui, Zaid and Choi, Yejin},
  journal={arXiv preprint arXiv:2507.14843},
  year={2025}
}

@article{shao2025spurious,
  title={Spurious rewards: Rethinking training signals in rlvr},
  author={Shao, Rulin and Li, Shuyue Stella and Xin, Rui and Geng, Scott and Wang, Yiping and Oh, Sewoong and Du, Simon Shaolei and Lambert, Nathan and Min, Sewon and Krishna, Ranjay and others},
  journal={arXiv preprint arXiv:2506.10947},
  year={2025}
}

@article{yeo2025demystifying,
  title={Demystifying long chain-of-thought reasoning in llms},
  author={Yeo, Edward and Tong, Yuxuan and Niu, Morry and Neubig, Graham and Yue, Xiang},
  journal={arXiv preprint arXiv:2502.03373},
  year={2025}
}

@article{yuan2025f,
  title={From $ f (x) $ and $ g (x) $ to $ f (g (x)) $: LLMs Learn New Skills in RL by Composing Old Ones},
  author={Yuan, Lifan and Chen, Weize and Zhang, Yuchen and Cui, Ganqu and Wang, Hanbin and You, Ziming and Ding, Ning and Liu, Zhiyuan and Sun, Maosong and Peng, Hao},
  journal={arXiv preprint arXiv:2509.25123},
  year={2025}
}

@article{sun2025rl,
  title={RL Grokking Recipe: How Does RL Unlock and Transfer New Algorithms in LLMs?},
  author={Sun, Yiyou and Cao, Yuhan and Huang, Pohao and Bai, Haoyue and Hajishirzi, Hannaneh and Dziri, Nouha and Song, Dawn},
  journal={arXiv preprint arXiv:2509.21016},
  year={2025}
}

@article{zhang2025interplay,
  title={On the Interplay of Pre-Training, Mid-Training, and RL on Reasoning Language Models},
  author={Zhang, Charlie and Neubig, Graham and Yue, Xiang},
  journal={arXiv preprint arXiv:2512.07783},
  year={2025}
}

@article{schulman2017proximal,
  title={Proximal policy optimization algorithms},
  author={Schulman, John and Wolski, Filip and Dhariwal, Prafulla and Radford, Alec and Klimov, Oleg},
  journal={arXiv preprint arXiv:1707.06347},
  year={2017}
}

@article{prabhudesai2025maximizing,
  title={Maximizing Confidence Alone Improves Reasoning},
  author={Prabhudesai, Mihir and Chen, Lili and Ippoliti, Alex and Fragkiadaki, Katerina and Liu, Hao and Pathak, Deepak},
  journal={arXiv preprint arXiv:2505.22660},
  year={2025}
}

@article{zuo2025ttrl,
  title={Ttrl: Test-time reinforcement learning},
  author={Zuo, Yuxin and Zhang, Kaiyan and Sheng, Li and Qu, Shang and Cui, Ganqu and Zhu, Xuekai and Li, Haozhan and Zhang, Yuchen and Long, Xinwei and Hua, Ermo and others},
  journal={arXiv preprint arXiv:2504.16084},
  year={2025}
}

@article{schulman2015high,
  title={High-dimensional continuous control using generalized advantage estimation},
  author={Schulman, John and Moritz, Philipp and Levine, Sergey and Jordan, Michael and Abbeel, Pieter},
  journal={arXiv preprint arXiv:1506.02438},
  year={2015}
}

@article{silver2017mastering,
  title={Mastering chess and shogi by self-play with a general reinforcement learning algorithm},
  author={Silver, David and Hubert, Thomas and Schrittwieser, Julian and Antonoglou, Ioannis and Lai, Matthew and Guez, Arthur and Lanctot, Marc and Sifre, Laurent and Kumaran, Dharshan and Graepel, Thore and others},
  journal={arXiv preprint arXiv:1712.01815},
  year={2017}
}

@article{schrittwieser2020mastering,
  title={Mastering atari, go, chess and shogi by planning with a learned model},
  author={Schrittwieser, Julian and Antonoglou, Ioannis and Hubert, Thomas and Simonyan, Karen and Sifre, Laurent and Schmitt, Simon and Guez, Arthur and Lockhart, Edward and Hassabis, Demis and Graepel, Thore and others},
  journal={Nature},
  volume={588},
  number={7839},
  pages={604--609},
  year={2020},
  publisher={Nature Publishing Group UK London}
}

@misc{stojanovski2025reasoninggymreasoningenvironments,
      title={REASONING GYM: Reasoning Environments for Reinforcement Learning with Verifiable Rewards},
      author={Zafir Stojanovski and Oliver Stanley and Joe Sharratt and Richard Jones and Abdulhakeem Adefioye and Jean Kaddour and Andreas Köpf},
      year={2025},
      eprint={2505.24760},
      archivePrefix={arXiv},
      primaryClass={cs.LG},
      url={https://arxiv.org/abs/2505.24760},
}

@misc{qwen3technicalreport,
      title={Qwen3 Technical Report}, 
      author={Qwen Team},
      year={2025},
      eprint={2505.09388},
      archivePrefix={arXiv},
      primaryClass={cs.CL},
      url={https://arxiv.org/abs/2505.09388}, 
}

@misc{qwen2.5,
    title = {Qwen2.5: A Party of Foundation Models},
    url = {https://qwenlm.github.io/blog/qwen2.5/},
    author = {Qwen Team},
    month = {September},
    year = {2024}
}

@misc{rllm2025,
  title={rLLM: A Framework for Post-Training Language Agents},
  author={Sijun Tan and Michael Luo and Colin Cai and Tarun Venkat and Kyle Montgomery and Aaron Hao and Tianhao Wu and Arnav Balyan and Manan Roongta and Chenguang Wang and Li Erran Li and Raluca Ada Popa and Ion Stoica},
  year={2025},
  howpublished={\url{https://pretty-radio-b75.notion.site/rLLM-A-Framework-for-Post-Training-Language-Agents-21b81902c146819db63cd98a54ba5f31}},
  note={Notion Blog}
}

@article{shao2024deepseekmath,
  title={Deepseekmath: Pushing the limits of mathematical reasoning in open language models},
  author={Shao, Zhihong and Wang, Peiyi and Zhu, Qihao and Xu, Runxin and Song, Junxiao and Bi, Xiao and Zhang, Haowei and Zhang, Mingchuan and Li, YK and Wu, Yang and others},
  journal={arXiv preprint arXiv:2402.03300},
  year={2024}
}

@book{sutton1998reinforcement,
  title={Reinforcement learning: An introduction},
  author={Sutton, Richard S and Barto, Andrew G and others},
  year={1998},
  publisher={MIT press Cambridge}
}

@article{kaelbling1996reinforcement,
  title={Reinforcement learning: A survey},
  author={Kaelbling, Leslie Pack and Littman, Michael L and Moore, Andrew W},
  journal={Journal of artificial intelligence research},
  volume={4},
  pages={237--285},
  year={1996}
}

@article{samineni2025rl,
  title={RL in Name Only? Analyzing the Structural Assumptions in RL post-training for LLMs},
  author={Samineni, Soumya Rani and Kalwar, Durgesh and Valmeekam, Karthik and Stechly, Kaya and Kambhampati, Subbarao},
  journal={arXiv preprint arXiv:2505.13697},
  year={2025}
}

@article{teng2025atom,
  title={Atom of thoughts for markov llm test-time scaling},
  author={Teng, Fengwei and Shi, Quan and Yu, Zhaoyang and Zhang, Jiayi and Luo, Yuyu and Wu, Chenglin and Guo, Zhijiang},
  journal={arXiv preprint arXiv:2502.12018},
  year={2025}
}

@article{viteri2024markovian,
  title={Markovian Transformers for Informative Language Modeling},
  author={Viteri, Scott and Lamparth, Max and Chatain, Peter and Barrett, Clark},
  journal={arXiv preprint arXiv:2404.18988},
  year={2024}
}

@article{xin2024atomr,
  title={Atomr: Atomic operator-empowered large language models for heterogeneous knowledge reasoning},
  author={Xin, Amy and Liu, Jinxin and Yao, Zijun and Lee, Zhicheng and Cao, Shulin and Hou, Lei and Li, Juanzi},
  journal={arXiv preprint arXiv:2411.16495},
  year={2024}
}

@article{zhou2024self,
  title={Self-discover: Large language models self-compose reasoning structures},
  author={Zhou, Pei and Pujara, Jay and Ren, Xiang and Chen, Xinyun and Cheng, Heng-Tze and Le, Quoc V and Chi, Ed and Zhou, Denny and Mishra, Swaroop and Zheng, Huaixiu Steven},
  journal={Advances in Neural Information Processing Systems},
  volume={37},
  pages={126032--126058},
  year={2024}
}

@article{aghajohari2025markovian,
  title={The Markovian Thinker: Architecture-Agnostic Linear Scaling of Reasoning},
  author={Aghajohari, Milad and Chitsaz, Kamran and Kazemnejad, Amirhossein and Chandar, Sarath and Sordoni, Alessandro and Courville, Aaron and Reddy, Siva},
  journal={arXiv preprint arXiv:2510.06557},
  year={2025}
}

@inproceedings{li2023compressing,
  title={Compressing context to enhance inference efficiency of large language models},
  author={Li, Yucheng and Dong, Bo and Guerin, Frank and Lin, Chenghua},
  booktitle={Proceedings of the 2023 conference on empirical methods in natural language processing},
  pages={6342--6353},
  year={2023}
}

@inproceedings{li2025prompt,
  title={Prompt compression for large language models: A survey},
  author={Li, Zongqian and Liu, Yinhong and Su, Yixuan and Collier, Nigel},
  booktitle={Proceedings of the 2025 Conference of the Nations of the Americas Chapter of the Association for Computational Linguistics: Human Language Technologies (Volume 1: Long Papers)},
  pages={7182--7195},
  year={2025}
}

@article{packer2023memgpt,
  title={MemGPT: Towards LLMs as Operating Systems},
  author={Packer, Charles and Wooders, Sarah and Lin, Kevin and Fang, Vivian and Patil, Shishir G and Stoica, Ion and Gonzalez, Joseph E},
  journal={arXiv preprint arXiv:2310.08560},
  year={2023}
}

@article{chhikara2025mem0,
  title={Mem0: Building production-ready ai agents with scalable long-term memory},
  author={Chhikara, Prateek and Khant, Dev and Aryan, Saket and Singh, Taranjeet and Yadav, Deshraj},
  journal={arXiv preprint arXiv:2504.19413},
  year={2025}
}

@article{xu2025mem,
  title={A-mem: Agentic memory for llm agents},
  author={Xu, Wujiang and Liang, Zujie and Mei, Kai and Gao, Hang and Tan, Juntao and Zhang, Yongfeng},
  journal={arXiv preprint arXiv:2502.12110},
  year={2025}
}

@inproceedings{zhong2024memorybank,
  title={Memorybank: Enhancing large language models with long-term memory},
  author={Zhong, Wanjun and Guo, Lianghong and Gao, Qiqi and Ye, He and Wang, Yanlin},
  booktitle={Proceedings of the AAAI Conference on Artificial Intelligence},
  volume={38},
  pages={19724--19731},
  year={2024}
}

@article{wu2025resum,
  title={ReSum: Unlocking Long-Horizon Search Intelligence via Context Summarization},
  author={Wu, Xixi and Li, Kuan and Zhao, Yida and Zhang, Liwen and Ou, Litu and Yin, Huifeng and Zhang, Zhongwang and Yu, Xinmiao and Zhang, Dingchu and Jiang, Yong and others},
  journal={arXiv preprint arXiv:2509.13313},
  year={2025}
}

@article{qiao2025webresearcher,
  title={Webresearcher: Unleashing unbounded reasoning capability in long-horizon agents},
  author={Qiao, Zile and Chen, Guoxin and Chen, Xuanzhong and Yu, Donglei and Yin, Wenbiao and Wang, Xinyu and Zhang, Zhen and Li, Baixuan and Yin, Huifeng and Li, Kuan and others},
  journal={arXiv preprint arXiv:2509.13309},
  year={2025}
}

@article{lu2025scaling,
  title={Scaling llm multi-turn rl with end-to-end summarization-based context management},
  author={Lu, Miao and Sun, Weiwei and Du, Weihua and Ling, Zhan and Yao, Xuesong and Liu, Kang and Chen, Jiecao},
  journal={arXiv preprint arXiv:2510.06727},
  year={2025}
}

@inproceedings{
yuan2025reinforce,
title={Reinforce {LLM} Reasoning through Multi-Agent Reflection},
author={Yurun Yuan and Tengyang Xie},
booktitle={Forty-second International Conference on Machine Learning},
year={2025},
url={https://openreview.net/forum?id=6k3oFS3Lbl}
}

@article{chen2025seed,
  title={Seed-prover 1.5: Mastering undergraduate-level theorem proving via learning from experience},
  author={Chen, Jiangjie and Chen, Wenxiang and Du, Jiacheng and Hu, Jinyi and Jiang, Zhicheng and Jie, Allan and Jin, Xiaoran and Jin, Xing and Li, Chenggang and Shi, Wenlei and others},
  journal={arXiv preprint arXiv:2512.17260},
  year={2025}
}

@article{hubert2025olympiad,
  title={Olympiad-level formal mathematical reasoning with reinforcement learning},
  author={Hubert, Thomas and Mehta, Rishi and Sartran, Laurent and Horv{\'a}th, Mikl{\'o}s Z and {\v{Z}}u{\v{z}}i{\'c}, Goran and Wieser, Eric and Huang, Aja and Schrittwieser, Julian and Schroecker, Yannick and Masoom, Hussain and others},
  journal={Nature},
  pages={1--3},
  year={2025},
  publisher={Nature Publishing Group UK London}
}

@article{hui2024qwen2,
  title={Qwen2. 5-Coder Technical Report},
  author={Hui, Binyuan and Yang, Jian and Cui, Zeyu and Yang, Jiaxi and Liu, Dayiheng and Zhang, Lei and Liu, Tianyu and Zhang, Jiajun and Yu, Bowen and Dang, Kai and others},
  journal={arXiv preprint arXiv:2409.12186},
  year={2024}
}

@article{sheng2024hybridflow,
  title   = {HybridFlow: A Flexible and Efficient RLHF Framework},
  author  = {Guangming Sheng and Chi Zhang and Zilingfeng Ye and Xibin Wu and Wang Zhang and Ru Zhang and Yanghua Peng and Haibin Lin and Chuan Wu},
  year    = {2024},
  journal = {arXiv preprint arXiv: 2409.19256}
}

@article{ziebart2008maximum,
  title={Maximum entropy inverse reinforcement learning},
  author={Ziebart, Brian D and Maas, Andrew L and Bagnell, J Andrew and Dey, Anind K},
  journal={AAAI Conference on Artificial Intelligence},
  volume={8},
  pages={1433--1438},
  year={2008}
}

@book{ziebart2010modeling,
  title={Modeling purposeful adaptive behavior with the principle of maximum causal entropy},
  author={Ziebart, Brian D},
  year={2010},
  publisher={Carnegie Mellon University}
}

@article{neu2017unified,
  title={A unified view of entropy-regularized markov decision processes},
  author={Neu, Gergely and Jonsson, Anders and G{\'o}mez, Vicen{\c{c}}},
  journal={arXiv preprint arXiv:1705.07798},
  year={2017}
}

@article{ouyang2022training,
  title={Training language models to follow instructions with human feedback},
  author={Ouyang, Long and Wu, Jeffrey and Jiang, Xu and Almeida, Diogo and Wainwright, Carroll and Mishkin, Pamela and Zhang, Chong and Agarwal, Sandhini and Slama, Katarina and Ray, Alex and others},
  journal={Advances in neural information processing systems},
  volume={35},
  pages={27730--27744},
  year={2022}
}

@article{mnih2015human,
  title={Human-level control through deep reinforcement learning},
  author={Mnih, Volodymyr and Kavukcuoglu, Koray and Silver, David and Rusu, Andrei A and Veness, Joel and Bellemare, Marc G and Graves, Alex and Riedmiller, Martin and Fidjeland, Andreas K and Ostrovski, Georg and others},
  journal={nature},
  volume={518},
  number={7540},
  pages={529--533},
  year={2015},
  publisher={Nature Publishing Group}
}

@article{kakade2001natural,
  title={A natural policy gradient},
  author={Kakade, Sham M},
  journal={Advances in neural information processing systems},
  volume={14},
  year={2001}
}

@inproceedings{cai2020provably,
  title={Provably efficient exploration in policy optimization},
  author={Cai, Qi and Yang, Zhuoran and Jin, Chi and Wang, Zhaoran},
  booktitle={International Conference on Machine Learning},
  pages={1283--1294},
  year={2020},
  organization={PMLR}
}

@inproceedings{schulman2015trust,
  title={Trust region policy optimization},
  author={Schulman, John and Levine, Sergey and Abbeel, Pieter and Jordan, Michael and Moritz, Philipp},
  booktitle={International conference on machine learning},
  pages={1889--1897},
  year={2015},
  organization={PMLR}
}

@inproceedings{geist2019theory,
  title={A theory of regularized markov decision processes},
  author={Geist, Matthieu and Scherrer, Bruno and Pietquin, Olivier},
  booktitle={International conference on machine learning},
  pages={2160--2169},
  year={2019},
  organization={PMLR}
}

@book{cesa2006prediction,
  title={Prediction, learning, and games},
  author={Cesa-Bianchi, Nicolo and Lugosi, G{\'a}bor},
  year={2006},
  publisher={Cambridge university press}
}

@article{hazan2016introduction,
  title={Introduction to online convex optimization},
  author={Hazan, Elad and others},
  journal={Foundations and Trends{\textregistered} in Optimization},
  volume={2},
  number={3-4},
  pages={157--325},
  year={2016},
  publisher={Now Publishers, Inc.}
}

@article{even2009online,
  title={Online Markov decision processes},
  author={Even-Dar, Eyal and Kakade, Sham M and Mansour, Yishay},
  journal={Mathematics of Operations Research},
  volume={34},
  number={3},
  pages={726--736},
  year={2009},
  publisher={INFORMS}
}

@inproceedings{cheng2022adversarially,
  title={Adversarially trained actor critic for offline reinforcement learning},
  author={Cheng, Ching-An and Xie, Tengyang and Jiang, Nan and Agarwal, Alekh},
  booktitle={International Conference on Machine Learning},
  pages={3852--3878},
  year={2022},
  organization={PMLR}
}

@article{xie2021bellman,
  title={Bellman-consistent pessimism for offline reinforcement learning},
  author={Xie, Tengyang and Cheng, Ching-An and Jiang, Nan and Mineiro, Paul and Agarwal, Alekh},
  journal={Advances in neural information processing systems},
  volume={34},
  pages={6683--6694},
  year={2021}
}

@article{ji2024self,
  title={Self-play with adversarial critic: Provable and scalable offline alignment for language models},
  author={Ji, Xiang and Kulkarni, Sanjeev and Wang, Mengdi and Xie, Tengyang},
  journal={arXiv preprint arXiv:2406.04274},
  year={2024}
}

@inproceedings{Kakade2002ApproximatelyOA,
  title={Approximately Optimal Approximate Reinforcement Learning},
  author={Sham M. Kakade and John Langford},
  booktitle={International Conference on Machine Learning},
  year={2002},
  url={https://api.semanticscholar.org/CorpusID:31442909}
}

@article{foster2025good,
  title={Is a Good Foundation Necessary for Efficient Reinforcement Learning? The Computational Role of the Base Model in Exploration},
  author={Foster, Dylan J and Mhammedi, Zakaria and Rohatgi, Dhruv},
  journal={arXiv preprint arXiv:2503.07453},
  year={2025}
}

@article{xie2024exploratory,
  title={Exploratory preference optimization: Harnessing implicit q*-approximation for sample-efficient rlhf},
  author={Xie, Tengyang and Foster, Dylan J and Krishnamurthy, Akshay and Rosset, Corby and Awadallah, Ahmed and Rakhlin, Alexander},
  journal={arXiv preprint arXiv:2405.21046},
  year={2024}
}

@article{yuan2025trajectory,
  title={Trajectory Bellman residual minimization: A simple value-based method for LLM reasoning},
  author={Yuan, Yurun and Chen, Fan and Jia, Zeyu and Rakhlin, Alexander and Xie, Tengyang},
  journal={arXiv preprint arXiv:2505.15311},
  year={2025}
}

\clearpage

\appendix
\crefalias{section}{appendix}
\crefalias{subsection}{appendix}
\onecolumn

\begin{center}
{\LARGE Appendix}
\end{center}

\section{Related Work}

\paragraph{Debate on RL Post-Training}
An active debate has emerged regarding whether RL can endow models with reasoning capabilities that fundamentally exceed those acquired during pre-training. A growing body of work suggests that RL primarily refines, reweights, or selectively amplifies reasoning patterns already latent in the base model, rather than inducing genuinely novel capabilities \citep{shao2024deepseekmath, yue2025does, wu2025invisible, shao2025spurious, yeo2025demystifying}. This view is further reinforced by recent self-improvement methods that eliminate environment interaction or external reward signals altogether, yet achieve comparable gains \citep{prabhudesai2025maximizing, zuo2025ttrl}, indicating that improvement often arises from internal redistribution of probability mass rather than exploratory discovery.

In contrast, works reporting emergent capabilities under RL typically depend on restrictive training designs, such as prerequisite domain knowledge \citep{yuan2025f}, carefully curated task difficulty \citep{zhang2025interplay}, or explicitly designed warm-up phases and partial reward shaping \citep{sun2025rl}. These mechanisms substantially constrain the optimization landscape and guide learning toward known solution manifolds, suggesting that the observed gains reflect controlled extrapolation within a narrow hypothesis space rather than the discovery of fundamentally new reasoning trajectories.

Furthermore, \citet{samineni2025rl} offers a critical reassessment of the prevailing ``history-as-state'' formulation in RL post-training. Combining theoretical arguments with empirical evidence, they show that mainstream RL post-training methods are effectively equivalent to outcome-conditioned supervised learning, implying that---under this formulation---RL has not meaningfully exceeded the representational or optimization capabilities of supervised learning.

\paragraph{Context Management}
A growing body of work addresses context-length explosion by reducing the amount of information carried forward, either through compression \citep{li2023compressing, li2025prompt} or by external memory mechanisms \citep{packer2023memgpt, chhikara2025mem0, xu2025mem, zhong2024memorybank}. These methods typically discard irrelevant details and condense salient information into compact summaries to operate over long contexts. In agentic AI, related context management techniques aim to prevent unbounded growth of interaction histories between an agent and its environment, thereby alleviating context limits and improving the stability of long-horizon training \citep{wu2025resum, lu2025scaling}.

Replacing the full history with a concise summary breaks the strict ``history-as-state'' formulation. However, it only \textit{appears} to yield a Markovian state---compression alone does not ensure the Markov property. A Markov state must be a sufficient statistic for optimal future control: any two histories mapped to the same state should induce identical conditional transition and reward distributions. Existing summarization-based approaches do not enforce this sufficiency, leaving the policy to implicitly learn equivalence across exponentially many trajectories. Moreover, without internalizing the environment’s transition dynamics, a summarization model may not reliably produce valid Markov states from action sequences.

\paragraph{Markov LLM Reasoning}
A few works mitigate reliance on historical information by decomposing problems into atomic reasoning steps and exploring Markovian reasoning processes \citep{xin2024atomr,teng2025atom,zhou2024self}. As an example, Atom of Thoughts (AOT) \citep{teng2025atom} proposes a test-time reasoning framework that iteratively transforms a problem into a sequence of answer-equivalent but progressively simpler subquestions. At each iteration, the current question is decomposed into a temporary dependency structure and then contracted into a new subquestion, which serves as a Markov state for the next step, eliminating reliance on historical reasoning traces. These work focuses on improving test-time abilities through planning and decomposition. 

Additionally, Markovian Thinker \citep{aghajohari2025markovian} structures LLM reasoning into fixed-size chunks, limiting the length of each reasoning step. They find that during RL training the policy learns to write a
textual state near the end of each chunk sufficient for seamless continuation of reasoning after reset. Although their work suggests a specific way to obtain fixed-size textual state, our work systematically showcases and analyzes the benefits of introducing Markov states without additional assumptions on the method of state generation. 

Furthermore, \citet{viteri2024markovian} propose a framework that aims at mitigating unfaithful chain-of-thought reasoning. It consists of a Chain-of-Thought (CoT) generator and a downstream policy that produces the final answer conditioned solely on the generated CoT, thereby treating the CoT as a load-bearing Markov state. 

Finally, WebResearcher \citep{qiao2025webresearcher} can be viewed as applying explicit Markov state estimation to DeepResearch, training a single model to predict both the next state and the next action. However, its primary contribution is a new training algorithm for DeepResearch agents, while our work presents a focused study on the impact of Markov property.

\section{Theoretical Analysis of Markovian Efficiency}

In this section, we provide theoretical analysis regarding the benefit of Markov states additional to \cref{sec:theory}.

\subsection{General Policy Optimization Protocol}

\begin{algorithm}
\caption{General Policy Optimization Protocol}
\label{alg:general_policy_optimization}
\begin{algorithmic}[1]
\STATE \textbf{Initialize:} Initial policy $\pi^{(1)}$ and total iterations $T$.
\FOR{$t = 1$ to $T$}
    \STATE Estimate approximate advantage function $\widehat{A}^{\pi^{(t)}}$ based on current policy $\pi^{(t)}$.
    \STATE Update policy to obtain $\pi^{(t+1)}$ by optimizing objective involving $\widehat{A}^{\pi^{(t)}}$:
    \\[0.5em]
    \quad $\pi^{(t+1)} \leftarrow \mathsf{Optimize}\left( \pi^{(t)}, \widehat{A}^{\pi^{(t)}} \right)$
    \\[0.5em]
\ENDFOR
\STATE \textbf{Return:} Sequence of policies $\pi^{(1)}, \pi^{(2)}, \dots, \pi^{(T+1)}$
\end{algorithmic}
\end{algorithm}

This policy optimization protocol encompasses many popular algorithms. For PPO, the function $\mathsf{Optimize}$ is defined as
\begin{align*}
\mathsf{Optimize}\left( \pi^\iter{t},\widehat A^{\pi^\iter{t}} \right) =\argmax_{\pi} \E_{\pi^\iter{t}} \Big[ \min\Big( \tfrac{\pi(a_h\mid s_h)}{\pi^\iter{t}(a_h\mid s_h)} \widehat A_h^{\pi^\iter{t}}(s_h,a_h),
\text{clip}\Big(\tfrac{\pi(a_h\mid s_h)}{\pi^\iter{t}(a_h\mid s_h)}, 1-\varepsilon, 1+\varepsilon\Big) \widehat A_h^{\pi^\iter{t}}(s_h,a_h) \Big) \Big].
\end{align*}

For GRPO, $\mathsf{Optimize}$ is
\begin{align*}
    \mathsf{Optimize}\left( \pi^\iter{t},\widehat A^{\pi^\iter{t}} \right) &=\argmax_{\pi} \mathbb{E}_{x \sim \rho, \{o^\iter{i}\}_{i=1}^G \sim \pi^\iter{t}(\cdot \mid x)} \Bigg[ \frac{1}{G} \sum_{i=1}^G \frac{1}{|o_i|} \sum_{h=1}^{|o_i|} \Bigg\{ \nonumber \\
    & \quad \min \left[ \frac{\pi(a_h^\iter{i} \mid s_h^\iter{i})}{\pi^\iter{t}(a_h^\iter{i} \mid s_h^\iter{i})} \widehat A^{\pi^\iter{t}}_\iter{i}, \text{clip}\left(\frac{\pi(a_h^\iter{i} \mid s_h^\iter{i})}{\pi^\iter{t}(a_h^\iter{i} \mid s_h^\iter{i})}, 1-\epsilon, 1+\epsilon \right) \widehat A^{\pi^\iter{t}}_\iter{i} \right] - \beta D_{\text{KL}}(\pi \| \piref) \Bigg\}\Bigg],
\end{align*}

\subsection{Discussion of \cref{asm:opt-error}}
\label{sec:opt-error-asm-discussin}
To argue the validity of \cref{asm:opt-error}, we consider the connection between the natural policy gradient \citep{kakade2001natural} and KL-regularized policy optimization in proximal policy optimization \citep{schulman2015trust,schulman2017proximal,cai2020provably}. For ease of presentation, we define $x_h$ to be the essential state at time step $h$ ($x_h = (s_1,a_1,\dotsc,a_{h-1})$ for the action-sequence-based learning and $x_h = \widehat s_h$ for approximate Markovian learning). With this notation, we consider the following RL objectives
\begin{align*}
\pi^\iter{t+1} \leftarrow \argmax_{\pi} &~ \sum_{h = 1}^{H} \E_{\pi^\iter{t}}\left[ \frac{\pi(a_{h} \mid x_{h})}{\pi^\iter{t}(a_{h} \mid x_{h})} \widehat A^{\pi^\iter{t}}(x_{h}, a_{h}) - \beta D_{\mathrm{KL}}(\pi(\cdot \mid x_{h}) \| \pi^\iter{t}(\cdot \mid x_{h})) \right].
\end{align*}
This objective can be viewed as the approximation of popular RL algorithms for LLMs (we omit clipping and multiple rollouts below, since our focus is on convergence behavior). We can rewrite that objective as
\begin{align*}
\pi^\iter{t+1} \leftarrow \argmax_{\pi} &~ \sum_{h = 1}^{H} \E_{x_h \sim \pi^\iter{t}}\left[ \widehat A^{\pi^\iter{t}}(x_{h}, \pi) - \beta D_{\mathrm{KL}}(\pi(\cdot \mid x_{h}) \| \pi^\iter{t}(\cdot \mid x_{h})) \right],
\end{align*}
where $A(x,\pi) \coloneqq \E_{a \sim \pi}[A(x,a)]$.
Then, we can easily verify that one global optimum of this objective is
\begin{align*}
\pi^\iter{t+1}(\cdot \mid x_{h}) \propto \pi^\iter{t}(\cdot \mid x_{h}) \cdot \exp\left(\frac{1}{\beta}\widehat A^{\pi^\iter{t}}(x_{h}, \cdot)\right),
\end{align*}
which implies \cref{asm:opt-error} with $\varepsilon_{\sf opt} = O(\sqrt{{1}/{T}})$. There is a rich literature studied or used this argument from online learning \citep{cesa2006prediction,even2009online,hazan2016introduction,neu2017unified,geist2019theory} to recent RL advances \citep{cai2020provably,xie2021bellman,cheng2022adversarially,ji2024self}.

\subsection{Proof of \cref{prop:non-markov} and \cref{prop:app-markov}}
\label{sec:proof}
We first introduce an episodic version of performance difference lemma \citep{Kakade2002ApproximatelyOA}. 
\begin{lemma}[Performance difference lemma]
    \label[lemma]{lemma:perf-diff}
    For any policy $\pi'$ and $\pi$, we have
    \[J(\pi')-J(\pi)=\sum_{h=1}^{H} \E_{(s_h,a_h)\sim d_h^{\pi'}}\left[A_h^\pi(s_h,a_h)\right] \]
\end{lemma}

We provide the proof for these performance guarantees below.
\begin{proof}[\cpfname{prop:non-markov}]
For each $\pi^\iter{t}$, we have 
\begin{align*}
&~ J(\pi^\star) - \frac{1}{T} \sum_{t=1}^{T} J(\pi^\iter{t})
\\
= &~ \frac{1}{T} \sum_{t=1}^{T} \sum_{h=1}^{H}\E_{\pi^\star}\left[ A^{\pi^\iter{t}}(s_1,a_1,\dotsc,a_{h}) \right]
\\
= &~ \underbrace{\frac{1}{T} \sum_{t=1}^{T} \sum_{h=1}^{H}\E_{\pi^\star}\left[ \widehat A^{\pi^\iter{t}}(s_1,a_1,\dotsc,a_{h}) \right]}_{\text{(I)}} + \underbrace{\frac{1}{T} \sum_{t=1}^{T} \sum_{h=1}^{H}\E_{\pi^\star}\left[ A^{\pi^\iter{t}}(s_1,a_1,\dotsc,a_{h}) - \widehat A^{\pi^\iter{t}}(s_1,a_1,\dotsc,a_{h}) \right]}_{\text{(II)}}.
\end{align*}
For a fixed $h$, let $x_h = (s_1, a_1, \dots, a_{h-1})$ and $\tau_h = (s_1,a_{1:h}) = (x_h, a_h)$. 
The first equation is a direct application of \cref{lemma:perf-diff}, treating $(s_1,a_1,\dotsc,a_{h-1})$ as the essential state $x_h$. 

Term (I) can be bounded as $\leq H \varepsilon_{\sf opt}$ by \cref{asm:opt-error}. 
We now bound the term (II). We have:
\begin{align*}
\E_{\pi^\star}\left[ A^{\pi^\iter{t}}(\tau_h) - \widehat A^{\pi^\iter{t}}(\tau_h) \right]
&= \sum_{\tau_h \in \Scal_1\times \Acal^h} d_h^{\pi^\star}(\tau_h) (A^{\pi^\iter{t}}(\tau_h) - \widehat A^{\pi^\iter{t}}(\tau_h))
\\
&= \sum_{\tau_h \in \Scal_1\times \Acal^h} d_h^{\pi^\iter{t}}(\tau_h) \frac{d_h^{\pi^\star}(\tau_h)}{d_h^{\pi^\iter{t}}(\tau_h)} (A^{\pi^\iter{t}}(\tau_h) - \widehat A^{\pi^\iter{t}}(\tau_h))
\\
&\le \sqrt{\E_{\pi^\iter{t}}\left[ \left(\frac{d_h^{\pi^\star}(\tau_h)}{d_h^{\pi^\iter{t}}(\tau_h)}\right)^2 \right]} \sqrt{\E_{\pi^\iter{t}}[(A^{\pi^\iter{t}}(\tau_h) - \widehat A^{\pi^\iter{t}}(\tau_h))^2]}
\\
&\le \sqrt{\E_{\pi^\iter{t}}\left[ \left(\frac{d_h^{\pi^\star}(\tau_h)}{d_h^{\pi^\iter{t}}(\tau_h)}\right)^2 \right]} \varepsilon_{\sf stat}.
\end{align*}
Therefore, 
\begin{align*}
    \text{Term (II)} \leq H \sqrt{\max_{t,h} \E_{\pi^\iter{t}}\left[ \left( \frac{d_h^{\pi^\star}(s_1,a_{1:h})}{d_h^{\pi^\iter{t}}(s_1,a_{1:h})} \right)^2 \right]}\varepsilon_{\sf stat}
\end{align*}

Summing up the bounds of term (I) and (II) concludes our proof. 
\end{proof}

\begin{proof}[\cpfname{prop:app-markov}]
For each $\pi^\iter{t}$, we have 
\begin{align*}
&~ J(\pi^\star) - \frac{1}{T} \sum_{t=1}^{T} J(\pi^\iter{t})
\\
= &~ \frac{1}{T} \sum_{t=1}^{T} \sum_{h=1}^{H}\E_{\pi^\star}\left[ A^{\pi^\iter{t}}(s_h,a_h) \right]
\\
= &~ \frac{1}{T} \sum_{t=1}^{T} \sum_{h=1}^{H}\E_{\pi^\star}\left[ A^{\pi^\iter{t}}(s_h,a_h) - A^{\pi^\iter{t}}(\widehat s_h,a_h) \right] + \frac{1}{T} \sum_{t=1}^{T} \sum_{h=1}^{H}\E_{\pi^\star}\left[ A^{\pi^\iter{t}}(\widehat s_h,a_h) \right]
\\
= &~ \underbrace{\frac{1}{T} \sum_{t=1}^{T} \sum_{h=1}^{H}\E_{\pi^\star}\left[ \widehat A^{\pi^\iter{t}}(\widehat s_h,a_h) \right]}_{\text{(I)}} + \underbrace{\frac{1}{T} \sum_{t=1}^{T} \sum_{h=1}^{H}\E_{\pi^\star}\left[ A^{\pi^\iter{t}}(\widehat s_h,a_h) - \widehat A^{\pi^\iter{t}}(\widehat s_h,a_h) \right]}_{\text{(II)}}
\\
&~ + \underbrace{\frac{1}{T} \sum_{t=1}^{T} \sum_{h=1}^{H}\E_{\pi^\star}\left[ A^{\pi^\iter{t}}(s_h,a_h) - A^{\pi^\iter{t}}(\widehat s_h,a_h) \right]}_{\text{(III)}},
\end{align*}
where the first equation is a direct application of \cref{lemma:perf-diff}. 

Term (I) can be bounded as $\leq H \varepsilon_{\sf opt}$ by \cref{asm:opt-error}. 

For term (III), recall that $\Pr(s_h \neq \widehat s_h) \le H \varepsilon_{\Pcal}$, and therefore term (III) is bounded as:
\begin{align*}
    \E_{\pi^\star}\left[ A^{\pi^\iter{t}}(s_h,a_h) - A^{\pi^\iter{t}}(\widehat s_h,a_h) \right] 
    &\le \Pr(s_h \neq \widehat s_h) \cdot \max_{s, \widehat s, a} |A^{\pi^\iter{t}}(s,a) - A^{\pi^\iter{t}}(\widehat s,a)|
    \\
    &\le (H \varepsilon_{\Pcal}) \cdot 2H = 2 H^2 \varepsilon_{\Pcal}.
\end{align*}
Next, for term (II), let $\Delta(s,a) = A^{\pi^\iter{t}}(s,a) - \widehat A^{\pi^\iter{t}}(s,a)$. This term can be bounded as:
\begin{align*}
& \E_{\pi^\star}\left[ \Delta(\widehat s_h,a_h) \right]
\\
= & \E_{\pi^\star}\left[ \E[\Delta(\widehat s_h,a_h) \mid s_h, a_h] \right] 
\\
= & \sum_{s,a} d^{\pi^\iter{t}}(s,a) \frac{d^{\pi^\star}(s,a)}{d^{\pi^\iter{t}}(s,a)} \E[\Delta(\widehat s_h,a_h) \mid s, a]
\\
\le & \sqrt{\E_{\pi^\iter{t}}\left[ \left(\frac{d^{\pi^\star}(s_h,a_h)}{d^{\pi^\iter{t}}(s_h,a_h)}\right)^2 \right]} \sqrt{\E_{\pi^\iter{t}}[(\E[\Delta(\widehat s_h,a_h) \mid s_h, a_h])^2]} \tag{Cauchy-Schwarz}
\\
\le & \sqrt{\E_{\pi^\iter{t}}\left[ \left(\frac{d^{\pi^\star}(s_h,a_h)}{d^{\pi^\iter{t}}(s_h,a_h)}\right)^2 \right]} \sqrt{\E_{\pi^\iter{t}}[\Delta(\widehat s_h,a_h)^2]} \tag{Jensen's inequality}
\\
= & \sqrt{\E_{\pi^\iter{t}}\left[ \left(\frac{d^{\pi^\star}(s_h,a_h)}{d^{\pi^\iter{t}}(s_h,a_h)}\right)^2 \right]} \varepsilon_{\sf stat}.
\end{align*}

Summing up the bounds of term (I), (II), and (III) concludes our proof. 

\end{proof}

\subsection{Illustration of the Theoretical Analysis with Combination Lock}
\label{sec:theory-comblock-comparison}
\begin{figure}[ht]
    \centering
    \begin{tikzpicture}[
        node distance=1.5cm and 1.5cm,
        state/.style={circle, draw, minimum size=0.6cm, inner sep=0pt, font=\small},
        fail/.style={circle, draw, fill=red!10, minimum size=0.6cm, inner sep=0pt, font=\small},
        goal/.style={circle, draw, double, fill=green!10, minimum size=0.6cm, inner sep=0pt, font=\small},
        every edge/.style={draw, ->, >=stealth}
    ]
    
    \node (root) at (0,0) [state] {$s_1$};
    
    \node (n1_correct) at (1.5, 0.8) [state] {$h_1^+$};
    \node (n1_wrong)   at (1.5, -0.8) [fail] {$h_1^-$};
    
    \node (n2_cc) at (3, 1.6) [state] {$h_2^{++}$};
    \node (n2_cw) at (3, 0.0) [fail] {$h_2^{+-}$};
    \node (n2_wc) at (3, -1.6) [fail] {$h_2^{-+}$};
    
    \draw (root) edge node[above, sloped, font=\scriptsize] {Correct} (n1_correct);
    \draw (root) edge node[below, sloped, font=\scriptsize] {Wrong} (n1_wrong);
    
    \draw (n1_correct) edge node[above, sloped, font=\scriptsize] {Correct} (n2_cc);
    \draw (n1_correct) edge node[below, sloped, font=\scriptsize] {Wrong} (n2_cw);
    \draw (n1_wrong)   edge[dashed] (n2_wc);

    \node at (4, 1.6) {$\dots$};
    \node at (4, -1.6) {$\dots$};

    \node at (2, -2.5) {\textbf{(a) Non-Markovian (History Tree)}};
    \node at (2, -3.0) {\small State space grows as $\Ocal(|\Acal|^H)$.};
    \node at (2, -3.4) {\small Exponential Coverage Needed.};

    \begin{scope}[xshift=7cm]
        \node (m_root) at (0,0) [state] {$s_1$};
        \node (m_1) at (1.5, 0) [state] {$s_2$};
        \node (m_2) at (3.0, 0) [state] {$s_3$};
        \node (m_dots) at (4.2, 0) {$\dots$};
        \node (m_goal) at (5.2, 0) [goal] {$s_H$};

        \draw (m_root) edge node[above, font=\scriptsize] {Correct} (m_1);
        \draw (m_1) edge node[above, font=\scriptsize] {Correct} (m_2);
        \draw (m_2) edge (m_dots);
        \draw (m_dots) edge (m_goal);

        \draw (m_root) edge[loop above] node[font=\tiny] {Wrong} (m_root);
        \draw (m_1) edge[bend left=45] node[below, font=\scriptsize] {Wrong} (m_root);
        \draw (m_2) edge[bend left=60] node[below, font=\scriptsize] {Wrong} (m_root);

        \node at (2.5, -2.5) {\textbf{(b) Markovian (State Chain)}};
        \node at (2.5, -3.0) {\small State space grows as $\Ocal(H)$.};
        \node at (2.5, -3.4) {\small Linear Coverage Sufficient.};
    \end{scope}

    \end{tikzpicture}
    \caption{Visual comparison of the state space complexity for the Combination Lock problem.}
    \label{fig:theory-comblock-comparison}
\end{figure}
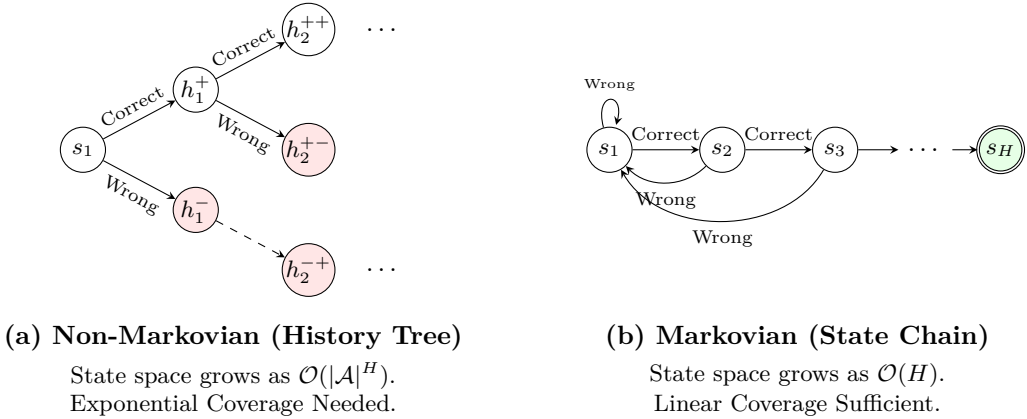

As shown in \cref{fig:theory-comblock-comparison}, the state space complexity diverges significantly depending on how the agent represents its environment. In the action-sequence-based learning, the agent treats the entire history of actions as its state, which can be visualized as a \emph{History Tree}. Each unique sequence of moves is treated as a distinct node in this tree. Consequently, the state space grows exponentially as $\Ocal(|\Acal|^H)$. To learn an optimal policy in this regime, the agent requires exponential coverage of the history space, leading to the prohibitively high sample complexity. 

In contrast, the Markovian formulation represents the environment as a \emph{State Chain}. By mapping all action histories to a Markov state $s_h$, the agent collapses the complex tree structure into a linear progression of length $H$. This representation ensures that the state space grows only linearly, $\Ocal(H)$, making linear coverage sufficient for convergence.

\subsection{Proofs for Supporting Lemmas}

\begin{proof}[\cpfname{lemma:perf-diff}]
Let us prove by induction. When $H = 1$, the trajectory consists of a single step, and the statement is trivial. 

Now assume that for horizon $k$ the statement holds for any policy $\pi'$ and $\pi$:
\[
J_k(\pi') - J_k(\pi) = \sum_{h=1}^k \mathbb{E}_{(s_h,a_h) \sim d_h^{\pi'}} \left[ A_h^\pi(s_h,a_h) \right].
\]
For $H = k + 1$, we expand $J_{k+1}(\pi') - J_{k+1}(\pi)$:
\begin{align*}
J_{k+1}(\pi') - J_{k+1}(\pi) =&~ \mathbb{E}_{\pi'} \left[ \sum_{h=1}^{k+1} r(s_h, a_h) \right] - \mathbb{E}_{\pi} \left[ \sum_{h=1}^{k+1} r(s_h, a_h) \right] \\
=&~ \mathbb{E}_{\pi'} \left[ \sum_{h=2}^{k+1} r(s_h, a_h) \right] - \E_{\pi'} [V_2^\pi (s_2)] \tag{I} \\
&~ + 
\E_{\pi'}[r(s_1,a_1)] + \E_{\pi'} [V_2^\pi (s_2)]  - \mathbb{E}_{\pi} \left[ \sum_{h=1}^{k+1} r(s_h, a_h) \right] \tag{II}
\end{align*}
Term I corresponds to the difference between expected returns of $\pi'$ and $\pi$ on an MDP with horizon $k$. Apply the inductive hypothesis to term I, we have
\begin{align*}
    \mathbb{E}_{\pi'} \left[ \sum_{h=2}^{k+1} r(s_h, a_h) \right] - \E_\pi' [V_2^\pi (s_2)]=\sum_{h=2}^H \mathbb{E}_{(s_h, a_h) \sim d_h^{\pi'}} \left[ Q_h^\pi(s_h, a_h) - V_h^\pi(s_h) \right]
\end{align*}
For term II, we have
\begin{align*}
&~ \E_{\pi'}[r(s_1,a_1)] + \E_{\pi'} [V_2^\pi (s_2)]  - \mathbb{E}_{\pi} \left[ \sum_{h=1}^{k+1} r(s_h, a_h) \right] \\
=&~ \E_{\pi'}[Q_1^\pi(s_1,a_1)]-\E_\pi[V_1^\pi(s_1)] \\
=&~ \E_{s_1\sim d_1^{\pi'}}\left[\E_{a_1\sim \pi'(\cdot\mid s_1)}\left[Q_1^\pi(s_1,a_1)\right]-V_1^\pi(s_1)\right]
\end{align*}

Summing up term I and II concludes our proof. 
\end{proof}

\section{Additional Results}
\label{sec:addition-results}
\subsection{Pass@\texorpdfstring{$k$}{k} Performance}
\label{sec:pass-k-full}
\begin{figure}
    \centering
    \includegraphics[width=1\linewidth]{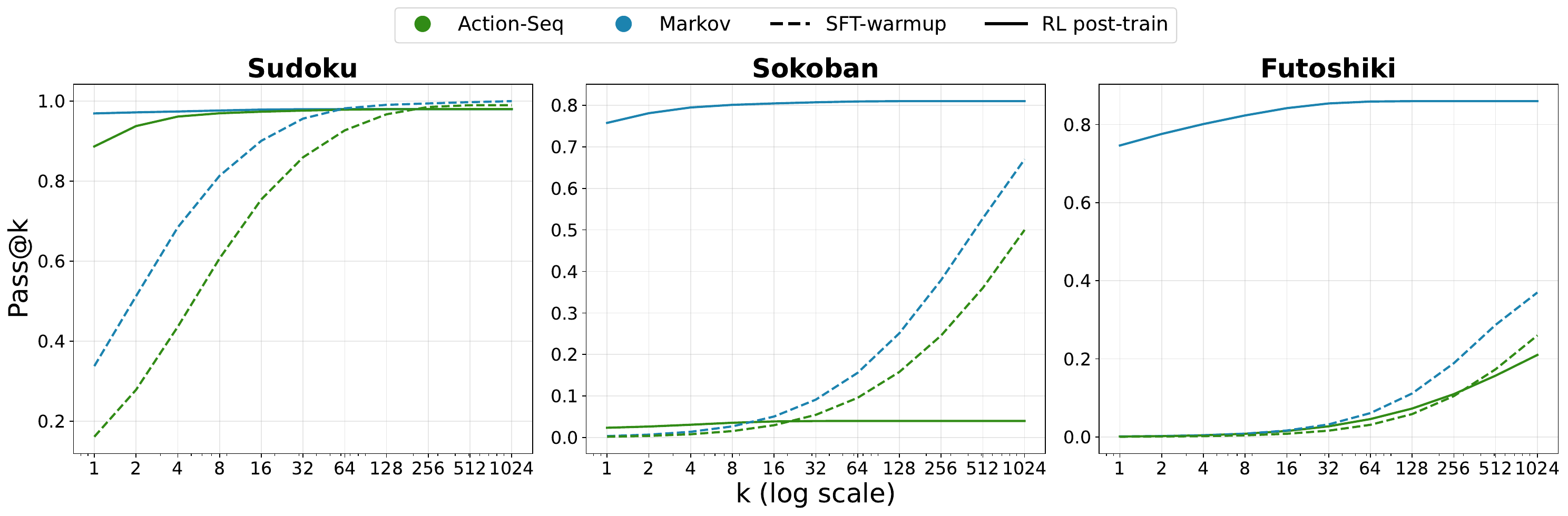}
    \caption{Pass@$k$ accuracy for Qwen3-4B-based models. While action-sequence models rarely improve SFT Pass@$k$, Markov models consistently surpass their base models' limits.}
    \label{fig:pass-at-k-all}
\end{figure}

We presents Qwen3-4B's Pass@$k$ accuracy as $k$ scales (\cref{fig:pass-at-k-all}). Several distinct patterns emerge. For less challenging tasks like Sudoku, where SFT-warmup models already achieve high Pass@$1024$, both Markov and action-sequence models sharpen Pass@$1$ performance, with Markov models holding a slight advantage. However, on more difficult tasks like Sokoban and Futoshiki, action-sequence models fail to extend or even maintain the Pass@$k$ of SFT models. In contrast, Markov models break through the capability boundaries of the base models, remarkably extending Pass@$k$.

\subsection{Training Success Rate}
\label{sec:train-success-rate}
We report the training-time success rate during RL post-training in \cref{fig:train-time-reward-complete}. For the experiments in \cref{sec:exp-markov-generalization}, the corresponding training dynamics are shown in \cref{fig:prm-train-time-reward-complete}. Across nearly all tasks and model variants, Markov models converge faster and achieve higher final success rates than action-sequence models. Introducing Markov states partially alleviates the slow-growth issue of action-sequence models in state–action-sequence variants, but a noticeable performance gap relative to Markov models remains.

\begin{figure}
    \centering
    \includegraphics[width=1\linewidth]{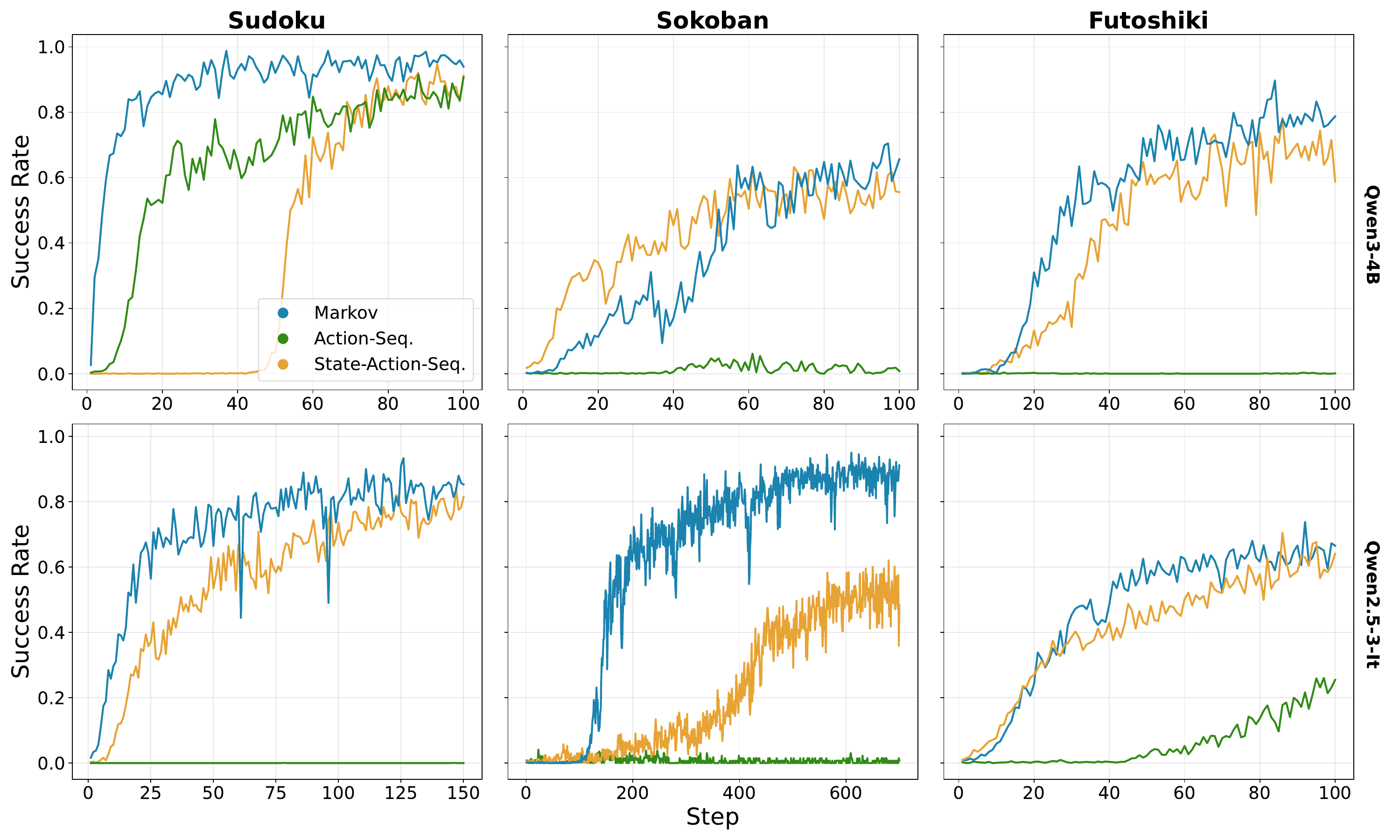}
    \caption{Training-time success rate during RL post-training. }
    \label{fig:train-time-reward-complete}
\end{figure}

\begin{figure}
    \centering
    \includegraphics[width=1\linewidth]{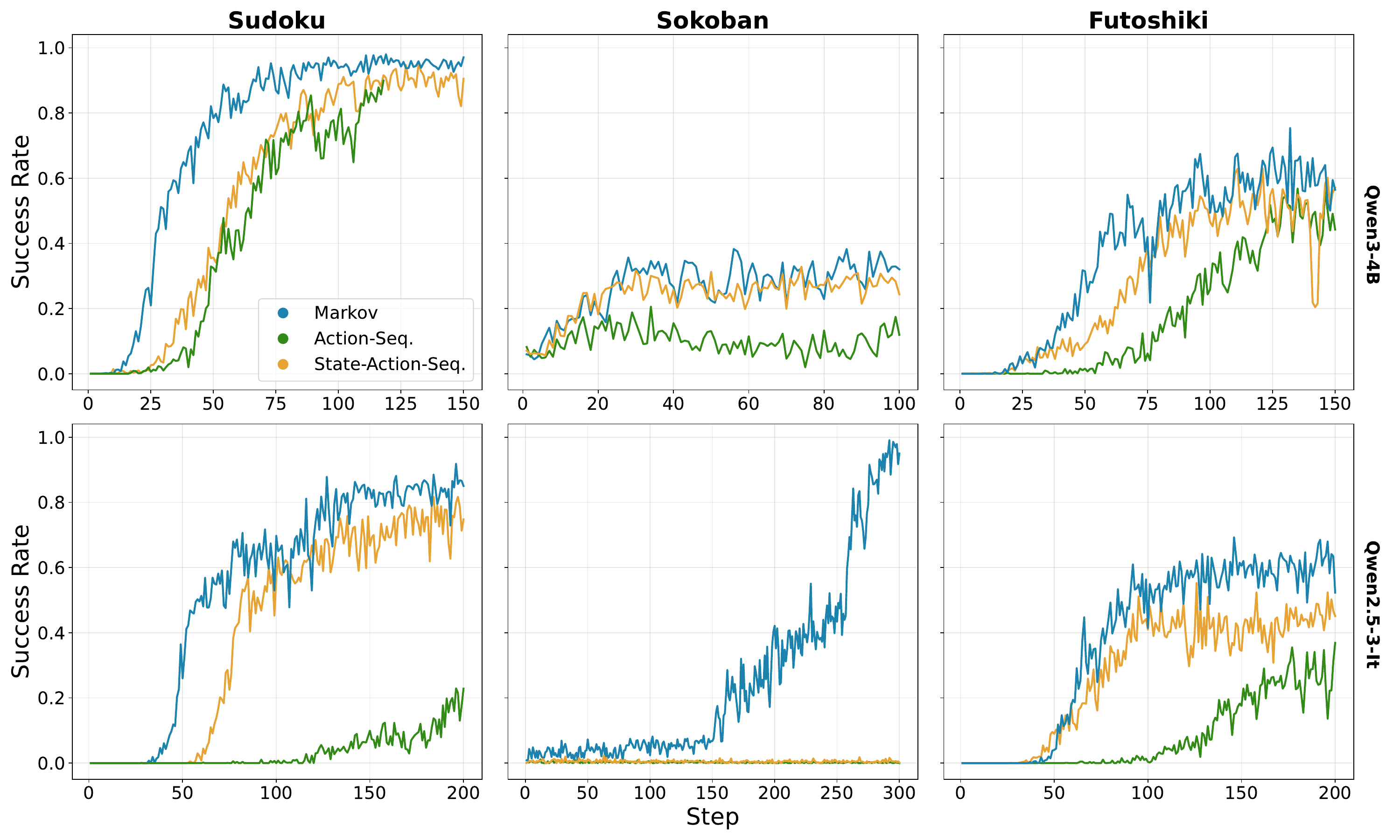}
    \caption{Training-time success rate during RL post-training for models with access to $\Astar$ in \cref{sec:exp-markov-generalization}. }
    \label{fig:prm-train-time-reward-complete}
\end{figure}

\subsection{Dependence on the Current State in State-Action-Sequence Models}
\label{sec:sas-model-depend-cur-state}

\begin{table}[ht]
    \centering
    \small
    \setlength{\tabcolsep}{8pt} %
    \begin{tabular}{lccc}
        \toprule
        \textbf{Approach} & \textbf{Sudoku} & \textbf{Sokoban} & \textbf{Futoshiki} \\
        \midrule
        \textbf{Qwen3-4B} & & & \\
        \textit{State-action-seq.} & 89.1 & 58.2 & 44.2 \\
        \hspace{0.5em} \textit{History-only} & 0.0 & 1.7 & 0.0 \\
        \hspace{0.5em} \textit{Last-state-only} & 39.2 & 3.1 & 26.6 \\
        \bottomrule
    \end{tabular}
    \caption{
        Analysis of the focus of the state-action-sequence models. \textbf{History-only}: the model is restricted to previous states and actions without access to the current state. \textbf{Last-state-only}: the model observes only the most recent state.         
        All evaluations are conducted with the ground-truth state transition $\mathcal{P}^\star$.
    }
    \label{tab:spurious-state}
\end{table}

To further probe what the state-action-sequence model actually uses after training, We test whether the model treats the current state as a sufficient statistic of the history. We evaluate two controlled variants: 
\begin{inlineenum}
    \item History-only, which removes the current state $s_h$ and provides only previous states and actions $\{(s_i,a_i)\}_{i=0}^{h-1}$, and
    \item Last-state-only, which keeps only the most recent state $s_h$ and discards the full history. 
\end{inlineenum}
\cref{tab:spurious-state} shows a stark pattern: History-only performance collapses to near zero across tasks, while Last-state-only retains a non-trivial fraction of the full state-action-sequence accuracy. Overall, these results indicate that well-trained state-action-sequence models rely primarily on the current state, with the action/state history contributing comparatively little, consistent with the view that they tend to internalize the Markov property of the state representation.

\subsection{Ablation on the Degree of SFT Warm-up}
\label{sec:ablation-sft-warmup-degree}

\begin{figure*}
    \centering
    \includegraphics[width=1\linewidth]{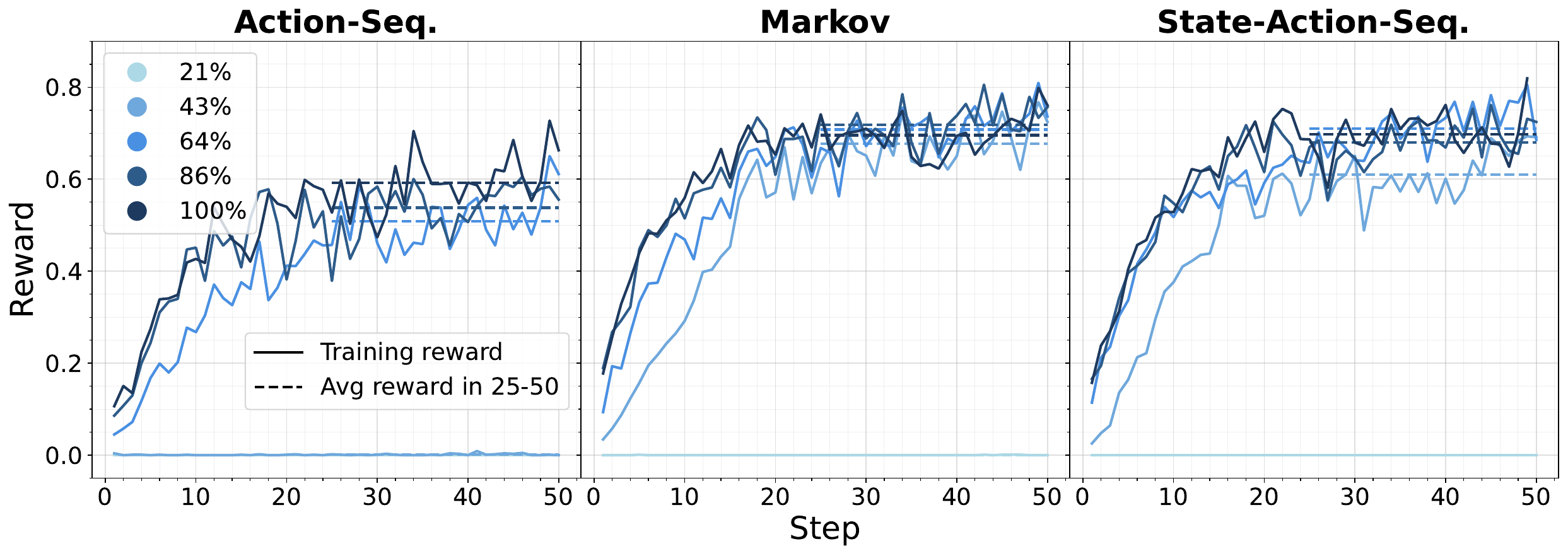}
    \caption{Ablation on the degree of SFT warm-up dataset. Experiments are conducted with Qwen2.5-3B-It on Sudoku. }
    \label{fig:sft-size-ablation}
\end{figure*}

To assess how the degree of SFT warm‑up affects subsequent RL, we run an ablation over the fraction of SFT steps used to initialize the policy. We first train a base model with SFT to completion, saving intermediate checkpoints at 21\%, 43\%, 64\%, 86\%, and 100\% of the total SFT steps. For each checkpoint, we then launch a fresh RL run, initializing the policy from that checkpoint while keeping all other RL settings fixed. As shown in \cref{fig:sft-size-ablation}, 21\% SFT is insufficient for all three approaches, with reward remaining at zero throughout RL. Continuing the warmup to 43\%, the reward of Markov and State-Action-Seq starts ascending, whereas the Action-Seq. model still fails to escape the zero‑reward plateau. Remarkably, even when fully SFTed (100\%), the Action-Seq. model never reaches the final reward achieved by the 43\%-SFT Markov model, highlighting that the Markov formulation is inherently easier for RL, more sample‑efficient, and less reliant on heavy SFT. Comparing Markov and State-Action-Seq., we observe that with 43\% SFT the latter learns more slowly and converges to slightly lower reward, but as the SFT fraction increases the State-Action-Seq. model gradually internalizes the Markov structure of the environment and matches the performance of the Markov baseline.

\subsection{Markov States in SFT}
\label{sec:markov-pure-sft}

\begin{table*}[th!]
    \centering
    \small
    \setlength{\tabcolsep}{0pt} 
    \begin{tabularx}{\textwidth}{l*{12}{>{\centering\arraybackslash}X}}
        \toprule
        \textbf{Approach} & \multicolumn{6}{c}{\textbf{In Distribution}} & \multicolumn{6}{c}{\textbf{Out Of Distribution}} \\
        \cmidrule(lr){2-7} \cmidrule(lr){8-13}
        & \multicolumn{2}{c}{\textbf{Sudoku}} & \multicolumn{2}{c}{\textbf{Sokoban}} & \multicolumn{2}{c}{\textbf{Futoshiki}} 
        & \multicolumn{2}{c}{\textbf{Sudoku}} & \multicolumn{2}{c}{\textbf{Sokoban}} & \multicolumn{2}{c}{\textbf{Futoshiki}} \\
        \cmidrule(lr){2-3} \cmidrule(lr){4-5} \cmidrule(lr){6-7} \cmidrule(lr){8-9} \cmidrule(lr){10-11} \cmidrule(lr){12-13}
        & \metricname{Avg} & \metricname{Pass} 
        & \metricname{Avg} & \metricname{Pass} 
        & \metricname{Avg} & \metricname{Pass} 
        & \metricname{Avg} & \metricname{Pass} 
        & \metricname{Avg} & \metricname{Pass} 
        & \metricname{Avg} & \metricname{Pass} \\
        \midrule
        \textbf{Qwen3-4B} & & & & & & & & & & & & \\
        \hspace{0.5em} \textit{Action-seq.} & 76.0 & 100.0 & 2.2 & 79.0 & 57.9 & 100.0 & 35.8 & 98.0 & 0.1 & 10.0 & 13.0 & 100.0 \\
        \rowcolor{highlightbgcolor}\hspace{0.5em} \textit{Markov} & 82.4 & 100.0 & 56.2 & 94.0 & 62.5 & 100.0 & 45.1 & 99.0 & 13.0 & 66.0 & 15.5 & 99.0 \\
        \hspace{0.5em} \textit{State-Action-seq.} & 70.4 & 100.0 & 53.6 & 98.0 & 61.0 & 100.0 & 33.1 & 99.0 & 14.9 & 58.0 & 14.8 & 100.0 \\
        \midrule
        \textbf{Qwen2.5-3B-It} & & & & & & & & & & & & \\
        \hspace{0.5em} \textit{Action-seq.} & 48.0 & 100.0 & 5.3 & 87.0 & 12.4 & 99.0 & 11.4 & 94.0 & 1.0 & 27.0 & 0.8 & 49.0 \\
        \rowcolor{highlightbgcolor}\hspace{0.5em} \textit{Markov} & 51.3 & 100.0 & 52.5 & 93.0 & 36.6 & 100.0 & 14.0 & 98.0 & 15.8 & 65.0 & 5.0 & 94.0 \\
        \hspace{0.5em} \textit{State-Action-seq.} & 62.0 & 100.0 & 35.7 & 95.0 & 50.3 & 100.0 & 19.9 & 97.0 & 5.5 & 58.0 & 9.4 & 98.0 \\
        \bottomrule
    \end{tabularx}
    \caption{Performance comparison of different approaches for SFT.}
    \label{tab:pure-sft}
\end{table*}

We have demonstrated the importance of Markov states in RL post-training due to its lower sample complexity, it is interesting to investigate the case in supervised learning. Therefore, we SFT the models and show the evaluation results in \cref{tab:pure-sft}. By comparing the approaches across models and tasks, we conclude two findings: 
\begin{inlineenum}
    \item Action-sequence models consistently performs worse than Markov models and state-action-sequence models, most significantly on Sokoban, indicating the benefits of Markov states even in the supervised learning paradigm. By conditioning on an explicit (predicted) state provided by an external transition model, models with Markov states no longer need to reconstruct the current board configuration implicitly in its latent space, thereby offloading the burden of state tracking and prediction, therefore lower the sample complexity.  
    \item The gap between Markov models and state-action-sequence models shrinks or becomes unclear. This is because the SFT objective is to maximize likelihood from fixed offline trajectories rather than exploratory discovery. Unlike RL, which requires expansive coverage to reliably explore, SFT bypasses the need for the significant sample complexity reduction. Consequently, the Markov property is less critical in supervised settings.
\end{inlineenum}

\subsection{Evaluation Results of Models with Access to \texorpdfstring{$\Astar$}{A*}}
\label{sec:eval-generalization-full}
\begin{table*}[ht!]
    \centering
    \small
    \begin{threeparttable}
    \setlength{\tabcolsep}{0pt} 
    \begin{tabularx}{\textwidth}{l *{12}{>{\centering\arraybackslash}X}} 
        \toprule
        \textbf{Approach} & \multicolumn{6}{c}{\textbf{In Distribution}} & \multicolumn{6}{c}{\textbf{Out Of Distribution}} \\
        \cmidrule(lr){2-7} \cmidrule(lr){8-13}
        & \multicolumn{2}{c}{\textbf{Sudoku}} & \multicolumn{2}{c}{\textbf{Sokoban}\tnote{1}} & \multicolumn{2}{c}{\textbf{Futoshiki}} 
        & \multicolumn{2}{c}{\textbf{Sudoku}} & \multicolumn{2}{c}{\textbf{Sokoban}\tnote{1}} & \multicolumn{2}{c}{\textbf{Futoshiki}} \\
        \cmidrule(lr){2-3} \cmidrule(lr){4-5} \cmidrule(lr){6-7} \cmidrule(lr){8-9} \cmidrule(lr){10-11} \cmidrule(lr){12-13}
        & \metricname{Avg} & \metricname{Pass} 
        & \metricname{Avg} & \metricname{Pass} 
        & \metricname{Avg} & \metricname{Pass} 
        & \metricname{Avg} & \metricname{Pass} 
        & \metricname{Avg} & \metricname{Pass} 
        & \metricname{Avg} & \metricname{Pass} \\
        \midrule
        \textbf{Qwen3-4B} & & & & & & & & & & & & \\
        \hspace{0.5em} \textit{Action-seq. w/ $A^*$} & 90.8 & 93.0 & 18.2 & 23.0 & 54.8 & 71.0 & 61.0 & 78.0 & 2.7 & 8.0 & 11.7 & 21.0 \\
        \rowcolor{highlightbgcolor}\hspace{0.5em} \textit{Markov w/ $A^*$} & 97.8 & 100.0 & 33.0 & 33.0 & 64.8 & 73.0 & 80.9 & 85.0 & 4.0 & 4.0 & 25.2 & 35.0 \\
        \hspace{0.5em} \textit{State-Action-seq. w/ $A^*$} & 94.0 & 98.0 & 19.8 & 25.0 & 51.6 & 69.0 & 73.7 & 83.0 & 3.9 & 8.0 & 14.4 & 32.0 \\
        \midrule
        \textbf{Qwen2.5-3B-It} & & & & & & & & & & & & \\
        \hspace{0.5em} \textit{Action-seq. w/ $A^*$} & 28.9 & 44.0 & 0.2 & 5.0 & 44.1 & 63.0 & 2.9 & 9.0 & 0.0 & 0.0 & 3.8 & 16.0 \\
        \rowcolor{highlightbgcolor}\hspace{0.5em} \textit{Markov w/ $A^*$} & 83.4 & 86.0 & 94.2 & 97.0 & 54.2 & 59.0 & 48.8 & 57.0 & 82.5 & 87.0 & 26.3 & 36.0 \\
        \hspace{0.5em} \textit{State-Action-seq. w/ $A^*$} & 84.8 & 92.0 & 0.2 & 1.0 & 51.9 & 71.0 & 45.8 & 59.0 & 0.0 & 0.0 & 11.2 & 26.0 \\
        \bottomrule
    \end{tabularx}
    \begin{tablenotes}
        \footnotesize
        \item[1] Given the capability constraints of the base models, we use lower-complexity Sokoban tasks in this section. 
    \end{tablenotes}
    \caption{Performance comparison of $\pimkv^{A^*}$, $\piactseq^{A^*}$, and $\pistactseq^{A^*}$ to show the difference of their generalization.}
    \label{tab:markov-astar-full}
    \end{threeparttable}
\end{table*}

We present the full evaluation results on $\pimkv^{\Astar}$, $\piactseq^{\Astar}$, and $\pistactseq^{\Astar}$ in \cref{tab:markov-astar-full} as a supplement to \cref{sec:exp-markov-generalization}.

\subsection{Examples of Solving Problems in Multiple Steps}
\definecolor{mydarkblue}{HTML}{0066CC}
\definecolor{mydarkgreen}{RGB}{0, 100, 0}

\newcommand{\bluebold}[1]{\textcolor{mydarkblue}{\textbf{#1}}}
\newcommand{\greenbold}[1]{\textcolor{mydarkgreen}{\textbf{#1}}}
\label{sec:examples}
\begin{tcolorbox}[breakable, colframe=gray, coltext=black, title=Action-Seq. Model $\piactseq$, fonttitle=\bfseries, boxrule=0.5mm, width=\textwidth]
\bluebold{State $s_0$}
\begin{adjustwidth}{2em}{0pt}
Here is the initial board state:\\
\texttt{3 6 4 8 1 9 2 \_ 7\\
5 9 8 3 7 2 6 1 4\\
\_ 2 \_ 5 6 4 9 3 8\\
8 1 5 4 3 6 7 9 2\\
9 \_ 6 2 8 7 5 4 1\\
2 4 \_ \_ 5 1 3 8 6\\
4 8 2 6 9 3 1 7 5\\
1 5 3 7 2 8 4 6 9\\
\_ 7 9 1 4 5 \_ 2 3\\}
\\
What is your immediate next move?
\end{adjustwidth}

\greenbold{Action $a_0$} 
\begin{adjustwidth}{2em}{0pt}
\textbackslash boxed\{0,7,5\}
\end{adjustwidth}

\bluebold{State $s_1$}
\begin{adjustwidth}{2em}{0pt}
Here is the initial board state:\\
\texttt{3 6 4 8 1 9 2 \_ 7\\
5 9 8 3 7 2 6 1 4\\
\_ 2 \_ 5 6 4 9 3 8\\
8 1 5 4 3 6 7 9 2\\
9 \_ 6 2 8 7 5 4 1\\
2 4 \_ \_ 5 1 3 8 6\\
4 8 2 6 9 3 1 7 5\\
1 5 3 7 2 8 4 6 9\\
\_ 7 9 1 4 5 \_ 2 3\\}
\\
Here is the history of moves made so far:\\
\textbf{Action taken: Fill cell at (0, 7) with 5}\\
What is your immediate next move?
\end{adjustwidth}
\greenbold{Action $a_1$}
\begin{adjustwidth}{2em}{0pt}
\textbackslash boxed\{8,0,6\}
\end{adjustwidth}

\bluebold{State $s_2$}
\begin{adjustwidth}{2em}{0pt}
Here is the initial board state:\\
\texttt{3 6 4 8 1 9 2 \_ 7\\
5 9 8 3 7 2 6 1 4\\
\_ 2 \_ 5 6 4 9 3 8\\
8 1 5 4 3 6 7 9 2\\
9 \_ 6 2 8 7 5 4 1\\
2 4 \_ \_ 5 1 3 8 6\\
4 8 2 6 9 3 1 7 5\\
1 5 3 7 2 8 4 6 9\\
\_ 7 9 1 4 5 \_ 2 3\\}
\\
Here is the history of moves made so far:\\
Action taken: Fill cell at (0, 7) with 5\\
\textbf{Action taken: Fill cell at (8, 0) with 6}\\
What is your immediate next move?
\end{adjustwidth}
$\cdots$

\end{tcolorbox}

\begin{tcolorbox}[breakable, colframe=gray, coltext=black, title=Markov Model $\pimkv$, fonttitle=\bfseries, boxrule=0.5mm, width=\textwidth]

\bluebold{State $s_0$}
\begin{adjustwidth}{2em}{0pt}
Current board state:\\
\texttt{3 6 4 8 1 9 2 \_ 7\\
5 9 8 3 7 2 6 1 4\\
\_ 2 \_ 5 6 4 9 3 8\\
8 1 5 4 3 6 7 9 2\\
9 \_ 6 2 8 7 5 4 1\\
2 4 \_ \_ 5 1 3 8 6\\
4 8 2 6 9 3 1 7 5\\
1 5 3 7 2 8 4 6 9\\
\_ 7 9 1 4 5 \_ 2 3\\}
\\
What is your immediate next move?
\end{adjustwidth}

\greenbold{Action $a_0$}
\begin{adjustwidth}{2em}{0pt}
\textbackslash boxed\{0,7,5\}
\end{adjustwidth}

\bluebold{State $s_1$}
\begin{adjustwidth}{2em}{0pt}
Current board state:\\
\texttt{3 6 4 8 1 9 2 5 7\\
5 9 8 3 7 2 6 1 4\\
\_ 2 \_ 5 6 4 9 3 8\\
8 1 5 4 3 6 7 9 2\\
9 \_ 6 2 8 7 5 4 1\\
2 4 \_ \_ 5 1 3 8 6\\
4 8 2 6 9 3 1 7 5\\
1 5 3 7 2 8 4 6 9\\
\_ 7 9 1 4 5 \_ 2 3\\}
\\
What is your immediate next move?
\end{adjustwidth}

\greenbold{Action $a_1$}
\begin{adjustwidth}{2em}{0pt}
\textbackslash boxed\{8,0,6\}
\end{adjustwidth}
$\cdots$

\end{tcolorbox}

\begin{tcolorbox}[breakable, colframe=gray, coltext=black, title=State-Action-Seq. Model $\pistactseq$, fonttitle=\bfseries, boxrule=0.5mm, width=\textwidth]

\bluebold{State $s_0$}
\begin{adjustwidth}{2em}{0pt}
Here is the current board state:\\
\texttt{3 6 4 8 1 9 2 \_ 7\\
5 9 8 3 7 2 6 1 4\\
\_ 2 \_ 5 6 4 9 3 8\\
8 1 5 4 3 6 7 9 2\\
9 \_ 6 2 8 7 5 4 1\\
2 4 \_ \_ 5 1 3 8 6\\
4 8 2 6 9 3 1 7 5\\
1 5 3 7 2 8 4 6 9\\
\_ 7 9 1 4 5 \_ 2 3\\}
\\
What is your immediate next move?
\end{adjustwidth}

\greenbold{Action $a_0$}
\begin{adjustwidth}{2em}{0pt}
\textbackslash boxed\{0,7,5\}
\end{adjustwidth}

\bluebold{State $s_1$}
\begin{adjustwidth}{2em}{0pt}
Here is the history of moves made so far:\\

State of the board:\\
\texttt{3 6 4 8 1 9 2 \_ 7\\
5 9 8 3 7 2 6 1 4\\
\_ 2 \_ 5 6 4 9 3 8\\
8 1 5 4 3 6 7 9 2\\
9 \_ 6 2 8 7 5 4 1\\
2 4 \_ \_ 5 1 3 8 6\\
4 8 2 6 9 3 1 7 5\\
1 5 3 7 2 8 4 6 9\\
\_ 7 9 1 4 5 \_ 2 3\\}
\\
\textbf{Action taken: Fill cell at (0, 7) with 5}\\

Here is the current board state:\\

\texttt{3 6 4 8 1 9 2 5 7\\
5 9 8 3 7 2 6 1 4\\
\_ 2 \_ 5 6 4 9 3 8\\
8 1 5 4 3 6 7 9 2\\
9 \_ 6 2 8 7 5 4 1\\
2 4 \_ \_ 5 1 3 8 6\\
4 8 2 6 9 3 1 7 5\\
1 5 3 7 2 8 4 6 9\\
\_ 7 9 1 4 5 \_ 2 3\\}
\\

What is your immediate next move?
\end{adjustwidth}
\greenbold{Action $a_1$}
\begin{adjustwidth}{2em}{0pt}
\textbackslash boxed\{8,0,6\}
\end{adjustwidth}

\bluebold{State $s_2$}
\begin{adjustwidth}{2em}{0pt}
Here is the history of moves made so far:\\

State of the board:\\
\texttt{3 6 4 8 1 9 2 \_ 7\\
5 9 8 3 7 2 6 1 4\\
\_ 2 \_ 5 6 4 9 3 8\\
8 1 5 4 3 6 7 9 2\\
9 \_ 6 2 8 7 5 4 1\\
2 4 \_ \_ 5 1 3 8 6\\
4 8 2 6 9 3 1 7 5\\
1 5 3 7 2 8 4 6 9\\
\_ 7 9 1 4 5 \_ 2 3\\}
\\
Action taken: Fill cell at (0, 7) with 5\\

State of the board:\\
\texttt{3 6 4 8 1 9 2 5 7\\
5 9 8 3 7 2 6 1 4\\
\_ 2 \_ 5 6 4 9 3 8\\
8 1 5 4 3 6 7 9 2\\
9 \_ 6 2 8 7 5 4 1\\
2 4 \_ \_ 5 1 3 8 6\\
4 8 2 6 9 3 1 7 5\\
1 5 3 7 2 8 4 6 9\\
\_ 7 9 1 4 5 \_ 2 3\\}
\\
\textbf{Action taken: Fill cell at (8, 0) with 6}\\

Here is the current board state:\\

\texttt{3 6 4 8 1 9 2 5 7\\
5 9 8 3 7 2 6 1 4\\
\_ 2 \_ 5 6 4 9 3 8\\
8 1 5 4 3 6 7 9 2\\
9 \_ 6 2 8 7 5 4 1\\
2 4 \_ \_ 5 1 3 8 6\\
4 8 2 6 9 3 1 7 5\\
1 5 3 7 2 8 4 6 9\\
6 7 9 1 4 5 \_ 2 3\\}
\\

What is your immediate next move?\\
\end{adjustwidth}
$\cdots$

\end{tcolorbox}

\subsection{Thinking Mode Hinders Models from Solving Problems in Multiple Steps}
\label{app:disable-think-mode}

As discussed in \cref{par:exp-impl-details}, we intentionally disable the model’s reasoning mode and require it to output only the discrete action, without test-time chain-of-thought. In this section, we use the following example to show that models with enabled thinking mode fails to decompose the problem into independent, progressive decision steps.

\begin{tcolorbox}[breakable, colframe=gray, coltext=black, title=Output of Qwen3-4B with thinking mode enabled, fonttitle=\bfseries, boxrule=0.5mm, width=\textwidth]

\textbf{System prompt}

You are playing Sudoku, a logic-based number placement puzzle.\\
\\
Rules:\\
- Fill the grid with numbers from 1 to 9.\\
- Each number must appear exactly once in each row, column, and 3x3 subgrid.\\
- Empty cells are represented by underscores (\_).\\
\\
Please provide only your next move in the format: \textbackslash boxed\{row, col, number\} where row and col are 0-indexed and number is between 1 and 9.\\

\textbf{User prompt}

Current board state:\\
\texttt{3 6 4 8 1 9 2 \_ 7\\
5 9 8 3 7 2 6 1 4\\
\_ 2 \_ 5 6 4 9 3 8\\
8 1 5 4 3 6 7 9 2\\
9 \_ 6 2 8 7 5 4 1\\
2 4 \_ \_ 5 1 3 8 6\\
4 8 2 6 9 3 1 7 5\\
1 5 3 7 2 8 4 6 9\\
\_ 7 9 1 4 5 \_ 2 3\\}
\\
What is your immediate next move?\\
\\
Please provide only your next move in the format: \textbackslash boxed\{row, col, number\} where row and col are 0-indexed and number is between 1 and 9.\\

\textbf{Model output}

\texttt{\textless think\textgreater}\\
Okay, let's try to figure out the next move for this Sudoku puzzle. First, I need to look at the current board state and see where I can find a cell that has only one possible number that can fit. Let me start by going through each row, column, and subgrid to check for any missing numbers.

$\cdots$

So existing numbers are 4,8,5,6,2,3,9. Missing are 1,7. Therefore, row 2 column 2 can be 1 or 7. Also, the subgrid (top-left) for row 2 column 2 is rows 0-2, columns 0-2. Existing numbers in that subgrid are 3,6,4 (row 0), 5,9,8 (row 1), and row 2: \_,2,\_. So existing numbers are 3,6,4,5,9,8,2. Missing are 1,7. Therefore, row 2 column 2 can be 1 or 7. So no unique move here. Therefore, not helpful.

\textbf{So the possible moves are row 0 column 7 (5), row 8 column 0 (6), row 8 column 6 (8), row 5 column 2 (7).} So which one is the next move? The problem is that the user may have a unique next move. But in the given puzzle, there might be multiple possibilities. However, the user wants the next move, so maybe the first one that is determined. Since row 0 column 7 is uniquely determined as 5, that's a valid move. Therefore, I think that is the correct answer.\\
\texttt{\textless/think\textgreater}

To determine the **immediate next move** in the given Sudoku puzzle, we examine the current state of the board for any cell that can be uniquely filled based on the rules of Sudoku: each number from 1 to 9 must appear exactly once in every row, column, and 3x3 subgrid.

---

\#\#\# Analyzing the Cell at Row 0, Column 7

The current value in the row is:

```
3 6 4 8 1 9 2 \_ 7
```

This is **Row 0**, and the empty cell is at **Column 7**. Let’s determine what number can go here.

- The existing numbers in **Row 0** are: 3, 6, 4, 8, 1, 9, 2, 7.
- The missing number in the row is **5**.
  
Now check if **5** is valid in **Column 7**:
- Column 7 values are: 1, 3, 9, 4, 8, 7, 6, 2 (from rows 1 to 8).
- So **5** is not in column 7 → it is **valid**.

Next, check the **subgrid** for **Row 0, Column 7** (the top-right 3x3 subgrid, rows 0–2, columns 7–8):
- Existing values: 7 (from row 0, col 8), 1 (row 1, col 7), 4 (row 1, col 8), 3 (row 2, col 7), 8 (row 2, col 8).
- No 5 is in the subgrid → valid.

Hence, **5 is the only number that can go into Row 0, Column 7**.

---

\#\#\# Final Answer

$$
\boxed{0, 7, 5}
$$

\end{tcolorbox}

This example illustrates the model's behavior as an implicit transition model; it forecasts future board states within its reasoning trace and solves the problem holistically rather than through decomposition.

\section{Implementation Details}
\label{sec:implementation-details}

\subsection{Datasets}
As discussed in \cref{par:exp-datasets}, we use three logical tasks during training and benchmarking. In this section, we provide details on these datasets for a more comprehensive setup description and reproducibility. 

\paragraph{Dataset Sizes}
As detailed in \cref{par:exp-impl-training}, we introduce a SFT warm-up phase before RL post-training. We use SFT dataset of different sizes to warm up different models on various tasks, due to the differences in base model capabilities and task complexities. The sizes are listed in \cref{tab:sft-warmup-dataset-size}. All models are trained for 1 epoch. 

\begin{table*}[ht]
    \centering
    \small
    \setlength{\tabcolsep}{15pt}
    \begin{tabular}{lccc} 
        \toprule
        \textbf{Model} & \textbf{Sudoku} & \textbf{Sokoban} & \textbf{Futoshiki} \\
        \midrule
        \textbf{Qwen3-4B} & 4,000 & 4,000 & 4,000 \\
        \textbf{Qwen2.5-3B-It} & 7,000 & 6,000 & 10,000 \\
        \bottomrule
    \end{tabular}
    \caption{SFT warm-up dataset sizes for each model on various tasks. }
    \label{tab:sft-warmup-dataset-size}
\end{table*}

In RL post-training phase, we use a dataset of 10,000 samples for Sudoku, 6,000 samples for Sokoban, and 6,000 samples for Futoshiki. We additionally synthesize 100 problems as test set for each task. For experiments in \cref{sec:markov-pure-sft}, we use larger SFT datasets than the one used in SFT warm-up, whose sizes are listed in \cref{tab:pure-sft-dataset-size}. All models are trained with 1 epoch, except that Qwen2.5-3B-It  is trained for 2 epochs on Sudoku dataset. 

\begin{table*}[ht]
    \centering
    \small
    \setlength{\tabcolsep}{15pt}
    \begin{tabular}{lccc} 
        \toprule
        \textbf{Model} & \textbf{Sudoku} & \textbf{Sokoban} & \textbf{Futoshiki} \\
        \midrule
        \textbf{Qwen3-4B} & 18,000 & 20,000 & 24,000 \\
        \textbf{Qwen2.5-3B-It} & 18,000$\times$2 & 24,000 & 32,000 \\
        \bottomrule
    \end{tabular}
    \caption{SFT dataset sizes for each model on various tasks used in \cref{sec:markov-pure-sft}. }
    \label{tab:pure-sft-dataset-size}
\end{table*}

In experiments conducted in \cref{sec:ablation-sft-warmup-degree}, we use the same dataset of size 18,000 training Qwen2.5-3B-It on Sudoku problems. 

\paragraph{Dataset Difficulties}
We synthesize the datasets using various configurations for different tasks. Details are presented in \cref{tab:dataset-config}. 
\begin{table*}[ht]
    \centering
    \small
    \begin{tabular*}{\textwidth}{@{\extracolsep{\fill}}lccc} 
        \toprule
        \textbf{Benchmark} & \textbf{Sudoku} & \textbf{Sokoban} & \textbf{Futoshiki} \\
        \midrule
        \textbf{Train set \& ID Tests} & \makecell[l]{Board size: $9\times 9$\\ Number of blanks: 6} & \makecell[l]{Board size: $9\times 9$\\ Minimal number of steps: 6--10} & \makecell[l]{Board size: $5\times 5$\\ Number of blanks: 8--10} \\
        \midrule
        \textbf{OOD Tests} & \makecell[l]{Board size: $9\times 9$\\ Number of blanks: 10} & \makecell[l]{Board size: $9\times 9$\\ Minimal number of steps: 12--14} & \makecell[l]{Board size: $5\times 5$\\ Number of blanks: 12--14} \\
        \bottomrule
    \end{tabular*}
    \caption{Configurations of different tasks used in training and testing.}
    \label{tab:dataset-config}
\end{table*}

An exception is the Sokoban dataset used in \cref{sec:exp-markov-generalization}. Given the capability constraints of the base models, we use lower-complexity Sokoban tasks. For in-distribution benchmarks, the minimal number of steps is between $2 \sim 4$; for out-of-distribution benchmarks, it is between $4 \sim 6$. 

\subsection{Training}
We use rLLM \citep{rllm2025} as the training and inference framework, which is backed by VERL \citep{sheng2024hybridflow}. We follow the most default hyperparameter during training. Specifically, we adopt a KL divergence coefficient of $0.001$, a learning rate of $1\times 10^{-6}$, a batch size of $128$, and sample $8$ responses for each group. The mini-batch size is set to 128, except for training Qwen2.5-3B-It on Sokoban after SFT warm-up, where a size of 64 is used.

For SFT warm-up stage, we use VERL \cite{sheng2024hybridflow} to train the models. The batch size is set to 256 and the learning rate is chosen as $5\times 10^{-6}$. 

\subsection{Training Details in \texorpdfstring{\cref{sec:exp-markov-generalization}}{Section \ref{sec:exp-markov-generalization}}}
\label{sec:prm-astar-comp}
In \cref{sec:exp-markov-generalization}, we use $\Astar$ to replace the estimated advantaged $\widehat A^\iter{i}$ in GRPO. Particularly, the orginal objective of GRPO is 

\begin{align}
    \mathcal{J}^{\sf GRPO}(\theta) &= \mathbb{E}_{x \sim \rho, \{o^\iter{i}\}_{i=1}^G \sim \pi_{\theta_{\sf old}}(\cdot \mid x)} \Bigg[ \frac{1}{G} \sum_{i=1}^G \frac{1}{|o_i|} \sum_{h=1}^{|o_i|} \Bigg\{ \nonumber \\
    & \quad \min \left[ \frac{\pi_\theta(a_h^\iter{i} \mid s_h^\iter{i})}{\pi_{\theta_{\sf old}}(a_h^\iter{i} \mid s_h^\iter{i})} \widehat A^\iter{i}, \text{clip}\left(\frac{\pi_\theta(a_h^\iter{i} \mid s_h^\iter{i})}{\pi_{\theta_{\sf old}}(a_h^\iter{i} \mid s_h^\iter{i})}, 1-\epsilon, 1+\epsilon \right) \widehat A^\iter{i} \right] - \beta D_{\text{KL}}(\pi \| \piref) \Bigg\}\Bigg],
    \label{eq:grpo-obj}
\end{align}

where $\widehat A^\iter{i} = \frac{r(x, o^\iter{i}) - \operatorname{mean}(\{r(x, o^\iter{1}), \dots, r(x, o^\iter{G})\})}{\operatorname{std}(\{r(x, o^\iter{1}), \dots, r(x, o^\iter{G})\})}$. Here, $a_h$ represents the $h$-th token and $s_h$ is the concatenation of all previous tokens. 

Replacing $\widehat A^\iter{i}$ with $\Astar$, the objective becomes

\begin{align*}
    \mathcal{J}^{\sf GRPO}_{\Astar}(\theta) &= \mathbb{E}_{x \sim \rho, \{o^\iter{i}\}_{i=1}^G \sim \pi_{\theta_{\sf old}}(\cdot \mid x)} \Bigg[ \frac{1}{G} \sum_{i=1}^G \frac{1}{|o_i|} \sum_{h=1}^{|o_i|} \Bigg\{ \\
    & \quad \min \left[ \frac{\pi_\theta(a_h^\iter{i} \mid s_h^\iter{i})}{\pi_{\theta_{\sf old}}(a_h^\iter{i} \mid s_h^\iter{i})} {\Astar}^\iter{i}, \text{clip}\left(\frac{\pi_\theta(a_h^\iter{i} \mid s_h^\iter{i})}{\pi_{\theta_{\sf old}}(a_h^\iter{i} \mid s_h^\iter{i})}, 1-\epsilon, 1+\epsilon \right) {\Astar}^\iter{i} \right] - \beta D_{\text{KL}}(\pi \| \piref) \Bigg\}\Bigg],
\end{align*}
where the advantage of the optimal policy ${\Astar}^\iter{i}$ is calculated by the single-step response $o_i$. 

Consider a response $o$. Let $\tilde{s}$ denote the current state (i.e., the board configuration for Markov models, and the history sequence for action-sequence and state--action-sequence models), and let $\tilde{a}$ denote the action represented by response $o$ (i.e., the row index, column index, and value to fill for Sudoku and Futoshiki; or the movement direction for Sokoban). The computation of $\Astar(\tilde{s}, \tilde{a})$ is task dependent.

For Sudoku and Futoshiki, we set the discount factor to $\gamma = 1$. For any valid board state $\tilde{s}$, the task is always solvable, and thus $V^\star(\tilde{s}) = 1$. If action $\tilde{a}$ transitions $\tilde{s}$ to another valid state $\tilde{s}'$---that is, the filled number does not violate any constraints and makes progress toward the final solution---then $Q^\star(\tilde{s}, \tilde{a}) = 1$; otherwise, $Q^\star(\tilde{s}, \tilde{a}) = 0$. Consequently, $\Astar(\tilde{s}, \tilde{a}) = 0$ if action $\tilde{a}$ can still lead to a correct final solution, and $-1$ otherwise.

For Sokoban, we use a discount factor of $\gamma = 0.5$. Let $n$ denote the minimum number of steps required to reach the goal from state $\tilde{s}$, and let $n'$ denote the minimum number of steps after taking action $\tilde{a}$. The value $n'$ may be infinite if $\tilde{a}$ leads to an unsolvable state. By definition, the optimal value and action-value functions are given by
\[
V^\star(\tilde{s}) = \gamma^{n-1}, \qquad
Q^\star(\tilde{s}, \tilde{a}) = \gamma^{n'-1}.
\]
Therefore,
\[
\Astar(\tilde{s}, \tilde{a}) = \gamma^{n'-1} - \gamma^{n-1}.
\]
In practice, we compute $n$ and $n'$ using breadth-first search.

\subsection{State Prediction Models}
We train a state prediction model $\pistatetrans$ based on Qwen2.5-3B-Instruct via SFT to predict the next state $s_{h+1}$ from the current state $s_h$ and action $a_h$. Particularly, we first collect triplets $(s, a, s')$ from the environment $\mathcal{P}^\star$, where $s'=\mathcal{P}(s,a)$. We then use $(s,a)$ as prompt and SFT the base model to predict $s'$. In practice, we use 174k samples to train a state prediction model for Sudoku, 91k samples for Sokoban, and 108k samples for Futoshiki. At test time, $\pistatetrans$ replaces the environment $\mathcal{P}^\star$, enabling deployment without environment access.

\end{document}